\documentclass[twoside]{article}

\usepackage[accepted]{aistats2015}
\usepackage{amsmath,graphicx}
\usepackage{subfigure} 
\usepackage{url}
\usepackage{algorithm}
\usepackage{algorithmic}
\usepackage{hyperref}
\usepackage{comment} 
\usepackage[round]{natbib}

\usepackage{amsmath,amsfonts,amssymb,amsthm}
\usepackage{bbm,bm,color}
\usepackage{multirow}
\usepackage{booktabs}

\newtheorem{theorem}{Theorem}
\newtheorem{lemma}{Lemma}

\newtheorem{definition}{Definition}
\newtheorem{proposition}{Proposition}

\DeclareMathOperator*{\argmin}{arg\,min}

\DeclareMathOperator{\eE}{\mathbb{E}}

\newcommand{\rR}{\mathbb{R}} 
\DeclareMathOperator{\diag}{diag}

\def\numofitems{{Q}}
\def\numofproj{{P}}
\definecolor{orange}{RGB}{1,1,0}

\begin{document}

%

%

\twocolumn[

\aistatstitle{A Topic Modeling Approach to Ranking}

\aistatsauthor{ Weicong Ding \And Prakash Ishwar \And
  Venkatesh Saligrama }

\aistatsaddress{Boston University \And Boston University \And
  Boston University } ]

\begin{abstract}
We propose a topic modeling approach to the prediction of preferences
in pairwise comparisons. We develop a new generative model for
pairwise comparisons that accounts for multiple shared latent 
rankings that are prevalent in a population of users. This new model
also captures inconsistent user behavior in a natural way. 
We show how the estimation of latent rankings in the new generative
model can be formally reduced to the estimation of topics in a
statistically equivalent topic modeling problem.
%
We leverage recent advances in the topic modeling literature to
develop an algorithm that can learn shared latent rankings with
provable consistency as well as sample and computational complexity
guarantees. We demonstrate that the new approach is empirically
competitive with the current state-of-the-art
approaches in predicting preferences on some semi-synthetic and real
world datasets.
\end{abstract}
\vspace*{-3ex}
\section{Introduction}
\label{sec:intro}
The recent explosion of web technologies has enabled us to collect an
immense amount of partial preferences for large sets of items, e.g.,
products from Amazon, movies from Netflix, or restaurants from Yelp,
from a large and diverse population of users through transactions,
clicks, check-ins, etc.~\citep[e.g.,][]{Lu11:ref, volkovs14a:ref,
  Agarwal14:ref}.
The goal of this paper is to develop a new approach to model, learn,
and ultimately predict the preference behavior of users in pairwise
comparisons which can form a building block for other partial
preferences.
%
%
Predicting preference behavior is important to personal recommendation systems, e-commerce, information retrieval, etc.

We propose a novel topic modeling approach to ranking and introduce a
new probabilistic generative model for pairwise comparisons that
accounts for a heterogeneous population of inconsistent users.
The essence of our approach is to view the outcomes of comparisons
generated by each user as a {\it probabilistic mixture} of a few {\it
  latent} global rankings that are {\it shared} across the
user-population.
This is especially appealing in the context of emerging web-scale
applications where
(i) there are multiple factors that influence individual preference
behavior, e.g., product preferences are influenced by price, brand,
etc.,
(ii) each individual is influenced by multiple latent factors to different extents, 
(iii) individual preferences for very similar items may be noisy and
change with time, and
(iv) the number of comparisons available from each user is typically
limited.
Research on ranking models to-date does not fully capture all these
important aspects.

In the literature, we can identify two categories of models.
In the first category of models the focus is on learning {\it one
  global} ranking that ``optimally'' agrees with the observations
according to some metric \citep[e.g.,][]{Gleich11:ref, Agarwal14:ref,
  volkovs14a:ref}. Loosely speaking, this tacitly presupposes a fairly
{\it homogeneous population} of users having very similar
preferences.
%
In the second category of models, there are multiple constituent
 rankings in the user population, but each user is associated
with a {\it single} ranking scheme sampled from a set of multiple
constituent rankings~\citep[e.g.,][]{Farias09:ref, Lu11:ref}.  Loosely
speaking, this tacitly presupposes a {\it heterogeneous population} of
users who are clustered into different types by their preferences
and whose preference behavior is influenced by only one factor.
In contrast to both these categories,
we model each user's pairwise preference behavior as a
mixed membership latent variable model. This captures both
heterogeneity (via the multiple shared constituent rankings) and
inconsistent preference behavior (via the probabilistic mixture). This
is a fundamental change of perspective from the traditional
clustering-based approach to a decomposition-based one.

A second contribution of this paper is the development of a novel
algorithmic approach to efficiently and consistently estimate the
latent rankings in our proposed model.
This is achieved by establishing a formal connection to probabilistic
topic modeling where each document in a corpus is viewed as a
probabilistic mixture of a few prevailing
topics \citep{Blei2012Review:ref}.
This formal link allows us to leverage algorithms that were
recently proposed in the topic modeling literature \citep{Arora2:ref,
  DDP:ref, Ding14:ref} for estimating latent shared rankings.
Overall, our approach has a running time and a sample complexity bound
that are provably {\it polynomial} in all model parameters. Our approach is
asymptotically consistent as the number of users goes to infinity even
when the number of comparisons for each user is a small constant.
We also demonstrate competitive empirical performance in collaborative
prediction tasks. Through a variety of performance metrics, we
demonstrate that our model can effectively capture the variability of
real-world user preferences.
\vspace*{-2ex}
\section{Related Work}
\vspace*{-2ex}
Rank estimation from partial or total rankings has been extensively
studied over the last several decades in various settings.
%
%
A prominent setting is one in which individual user rankings (in a homogeneous population) are modeled as independent drawings from a
probability distribution which is centered around a
single ground-truth global ranking.
Efficient algorithms have been developed to estimate the global
ranking under a variety of probability models
\cite[][]{Quin10:ref,Gleich11:ref,Negahban12:ref,Osting13:ref,volkovs14a:ref}.
Chief among them are the Mallows model \citep{mallows1957:ref}, the
Plackett-Luce (PL) model \citep{plackett1975analysis}, and the Bradly-Terry-Luce (BTL) model \citep{Agarwal14:ref}.

To account for the heterogeneity in the user population,
\citep{Jagabathula08:ref, Farias09:ref} considered
models with multiple prevalent rankings and proposed consistent
combinatorial algorithms for estimating the rankings. The mixture of
Mallows model recently studied in \citep{Lu11:ref, MxMallow14:ref}
considers multiple constituent rankings as the ``centers'' for the Mallows
components, as do the ``mixture of PL'' and the ``mixture of BTL'' models \citep{MxRUT13:ref, MxMNL14:ref}.
In all these settings, however, each user is associated with only one
ranking sampled from the mixture model. 
%
They capture the cases where the population can be clustered into a
few types
in terms of their preference behavior.

The setup of our model, although being fundamentally different in
modeling perspective, is most closely related to the seminal work in
\cite{Jagabathula08:ref, Farias09:ref} (denoted by FJS) (see
Table~\ref{table:compare} and appendix).
As it turns out, our proposed model subsumes those proposed in FJS as
special cases.
On the other hand, while the 
algorithm in FJS can be applied to our more general setting, our
algorithm has provably better computational efficiency, polynomial
sample complexity, and superior empirical performance.
\begin{table*}[t]
\caption{\small Comparison to closely related work \citep{Jagabathula08:ref}
  \citep{Farias09:ref} (FJS) }
\label{table:compare}
\centering
{\small
\begin{tabular}{|c|c|c|c|c|c|}
\hline 
{\bf Method} & {\bf Assumptions} & {\bf Statistics} & {\bf
  Consistency} & {\bf Computational} & {\bf Sample} \\
& {\bf on} $\bm{\sigma}$ & {\bf used} &{\bf proved?}& {\bf
  complexity}& {\bf complexity}\\
\hline 
FJS & Separability & 1st order & Yes & Exponential in $K$ & Not
provided \\
\hline 
This paper & Separability & up to 2nd order & Yes & Polynomial &
Polynomial \\
\hline 
\end{tabular} 
}
\vspace*{-2ex}
\end{table*}

{\bf Relation to topic modeling:} Our ranking model shares the same
motivation as topic models. Topic modeling has been extensively
studied over the last decade and has yielded a number of powerful
approaches \citep[e.g.,][]{Blei2012Review:ref}.
While the dominant trend is to fit a MAP/ML estimate using
approximation heuristics such as variational Bayes or MCMC, recent
work has demonstrated that the topic discovery problem can lend itself
to provably efficient solutions with additional structural
conditions~\citep{Arora2:ref, Ding14:ref}. This forms the basis of our
technical approach.

{\bf Relation to rating based methods:} There is also a considerable
body of work on modeling numerical ratings
\citep[e.g.,][]{recsysbook:ref} from which ranking preferences can be
derived.
An emerging trend explores the idea of combining a topic model for
{\it text reviews} simultaneously with a rating-based model for ``star
ratings'' \citep{WongBlei11:ref}.
These approaches are, however, outside the scope of this paper.

The rest of the paper is organized as follows. We formally introduce
the new generative model in Sec.~\ref{sec:model}.
We then present the key geometrical perspective underlying the
proposed approach in Sec.~\ref{sec:idea}. We summarize the main steps
of our algorithm and the overall computational and statistical
efficiency in Sec.~\ref{sec:alg}.
We demonstrate competitive performance on semi-synthetic and
real-world datasets in Sec.~\ref{sec:exp}.
\vspace*{-2ex}
\section{A new generative model}
\vspace*{-2ex}
To formalize our proposed model, let $\mathcal{U} :=
\{1,\ldots,\numofitems\}$ be a universe of $\numofitems$ items. Let
the $K$ latent rankings over $\numofitems$ items that are shared
across a population of $M$ users be denoted by permutations
$\sigma^1,\ldots, \sigma^K$. Each user compares $N\geq 2$ pairs of items. The {\it unordered} item pairs $\{i,j\}$ to be compared are assumed to be drawn independently from some distribution $\mu$ with $\mu_{i,j}>0$ for all $i,j$ pairs. The $n$-th comparison result of user $m$ is denoted by an {\it ordered} pair $w_{m,n} = (i,j)$, if user $m$ compares item $i$ and
$j$ and prefers $i$ over $j$. Let a probability vector $\bm{\theta}_{m}$ be the user-specific weights over the $K$ latent rankings. The generative model for the comparisons from each user $m=1,\ldots, M$ is,
\vspace*{-1ex}
\begin{enumerate}
\item Sample $\bm{\theta}_{m} \in \bigtriangleup^{K}$ from a
  prior distribution $\Pr(\theta)$
\vspace*{-2ex}
\item For each comparison $n=1,\ldots, N$:
\vspace*{-1ex}
\begin{enumerate}
\item Sample a pair of items $\{i,j\}$ from $\mu$
\item Sample a ranking token $z_{m,n}\in\{1,\ldots, K\} \sim
  \text{Multinomial}(\bm\theta_m)$
\item If $\sigma^{z_{m,n}}(i) < \sigma^{z_{m,n}}(j)$, then $w_{m,n} =
  (i,j)$, otherwise $w_{m,n} = (j,i)$ \footnote {$\sigma^k(i)$ is the position of item  $i$ in the ranking $\sigma_k$ and item $i$ is preferred over $j$ if  $\sigma^k(i) < \sigma^k(j)$.} 
\end{enumerate}
\end{enumerate}
\vspace*{-1ex}

\label{sec:model} 
\begin{figure}[!htb]
\vglue -2ex
\centering{
\includegraphics[width=0.8\linewidth]{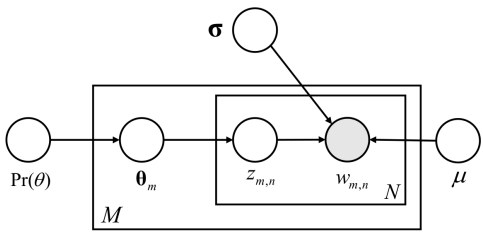}}
\vglue -3ex
\caption{{\small Graphical model representation of the generative model. The boxes represent replicates. The outer plate represents users, and the inner plate represents ranking tokens and comparisons of each user. }}
\label{fig:graphical}
\end{figure}
Figure~\ref{fig:graphical} is a standard graphical model
representation of the proposed generative process. Each user is
characterized by $\bm{\theta}_m$, the user-specific weights over the
$K$ shared rankings.
For convenience, we represent $\sigma^1,\ldots, \sigma^K$ by a
$W\times K$ nonnegative ranking matrix $\bm{\sigma}$ whose $W= Q(Q-1)$
rows are indexed by all the ordered pairs $(i,j)$.
We set $\sigma_{(i,j),k}= \mathbb{I}( \sigma^{k}(i)<\sigma^{k}(j) )$,
so that the $k$-th column of $\bm{\sigma}$ is an equivalent
representation of the ranking $\sigma^k$.
We then denote by $\bm{\theta}$ the $K\times M$ dimensional weight
matrix whose columns are the user-specific mixing weights
$\bm{\theta}_m$'s.
Finally, let $\mathbf{X}$ be the $W\times M$ empirical
comparisons-by-user matrix where $X_{(i,j),m}$ denotes the number of
times that user $m$ compares pair $\{i,j\}$ and prefers item $i$ over
$j$.
The principal algorithmic problem is to estimate the ranking matrix
$\bm{\sigma}$ given $\mathbf{X}$ and $K$.

If we denote by $\mathbf{P}$ a $W\times W$ diagonal matrix with the
$(i,j)$-th diagonal component $P_{(i,j),(i,j)} = \mu_{i,j}$, and set
$\mathbf{B} = \mathbf{P}\bm{\sigma}$, then the generative model
induces the following probabilities on comparisons $w_{m,n}$:
\vspace*{-1ex}
\begin{align}
\label{eq:distribution}
p(w_{m,n} = (i,j) \vert \bm{\theta}_m, \bm{B}) & = \mu_{i,j}
\sum\limits_{k=1}^{K}\sigma_{(i,j),k}\theta_{k,m} \nonumber \\
& = \sum_{k=1}^{K} B_{(i,j),k}\theta_{k,m}
\end{align}
Similarly, if we consider a probabilistic topic model on a set of $M$
documents, each composed of $N$ words drawn from a vocabulary of size
$W$, with a $W\times K$ {\it topic} matrix $\bm{\beta}$ and
document-specific mixing weights $\bm{\theta}^{\text{TM}}_m$ sampled
from a topic prior $\Pr^{\text{TM}}(\theta)$
\citep[e.g.][]{Blei2012Review:ref}, then, the distribution induced on
the observation $w_{m,n}^{\text{TM}}$, i.e., the $n$-th word in
document $m$, has the same form as in \eqref{eq:distribution}:
\vspace*{-1ex}
\begin{equation}
p(w_{m,n}^{\text{TM}} = i\vert \bm{\theta}^{\text{TM}}_m, \bm{\beta})
= \sum_{k=1}^{K}\beta_{i,k}\theta^{\text{TM}}_{k,m}
\label{eq:topicmodel}
\end{equation}
where $i=1,\ldots, W$ is any distinct word in the vocabulary. Noting
that $\mathbf{B}$ is column-stochastic, we have,
\begin{lemma}
\label{lem:statequivalent}
The proposed generative model is statistically equivalent to a
standard topic model whose topic matrix $\bm{\beta}$ is set to be $
\mathbf{B}$ and the topic prior to be $\Pr(\theta)$.
\end{lemma}
\begin{proof}  
Note that since $\mathbf{B}$ is column stochastic, it is a valid topic
matrix.
We need to show that the distribution on the comparisons
$\mathbf{w}=\{w_{m,n}\} $ and on the words in topic model
$\mathbf{w}^{\text{TM}}=\{w^{\text{TM}}_{m,n}\}$ are the same. From
\eqref{eq:distribution} \eqref{eq:topicmodel},
\vspace*{-2ex}
\begin{align*}
p(\mathbf{w} \vert \mathbf{B} ) & = \prod_{m=1}^{M} \int
p(w_{m,1},\ldots, w_{m,N}\vert \bm{\theta}_m, \mathbf{B})
\Pr(\bm{\theta}_m) d \bm{\theta}_m \\
& = \prod_{m=1}^{M} \int \left( \prod_{n=1}^{N}
\sum_{k=1}^{K}B_{w_{m,n},k}\theta_{k,m} \right) \Pr(\bm{\theta}_m) d
\bm{\theta}_m \\
%
%
&= p(\mathbf{w}^{\text{TM}} \vert \bm{\beta}).
\end{align*}
\vspace*{-2ex}
\end{proof}
Note that ${\bm{B}}=\mathbf{P} \bm{\sigma}$, $\mu_{i,j} = \mu_{j,i}$,
and ${\sigma}_{(i,j),k}+{\sigma}_{(j,i),k} = 1$. Hence $\bm{\sigma}$
can be inferred directly from $\mathbf{B}$:
\begin{equation}
\label{eq:sigma}
\sigma_{(i,j),k} = \frac{\sigma_{(i,j),k}\mu_{i,j} }{(
  \sigma_{(i,j),k}+ \sigma_{(j,i),k})\mu_{i,j}} = \frac{
  {B_{(i,j),k}}}{ B_{(i,j),k}+B_{(j,i),k}}
\end{equation}
Thus, the problem of estimating the ranking matrix $\bm{\sigma}$ can
be solved by any approach that can learn the topic matrix
$\bm{\beta}$.
Our approach is to leverage recent works in topic modeling
\citep{A12:ref,Arora2:ref,DDP:ref, Ding14:ref} that come with
consistency and statistical and computational efficiency guarantees by exploiting the second-order moments of the columns of $\mathbf{X}$, i.e., a  co-occurrence matrix of pairwise comparisons. 
We can establish parallel results for ranking model via the equivalency result of Lemma~\ref{lem:statequivalent}. 
In particular, by combining Lemma~\ref{lem:statequivalent} with results
in \citep[][Lemma 1 in Appendix]{DDP:ref}, the following result can be
immediately established:
\begin{lemma}
\label{lem:2ndOrder1}
If $\widetilde{\mathbf{X}}$ and $\widetilde{\mathbf{X}}^{\prime}$ are
obtained from $\mathbf{X}$ by first splitting each user's comparisons
into two independent copies and then re-scaling the rows to make them
row-stochastic, then
\vspace*{-1ex}
\begin{equation}
M \widetilde{\mathbf{X}}^{\prime} \widetilde{\mathbf{X}}^{\top}
\xrightarrow[\mbox{almost surely}]{M \rightarrow\infty}
\bar{\bm{B}}\bar{\mathbf{R}} \bar{\bm{B}}^{\top} =: \mathbf{E},
\vspace*{-1ex}
\end{equation}
where 
$\bar{\bm{B}} = \diag^{-1}(\bm{B}\mathbf{a})\bm{B}\diag(\mathbf{a})$,
$\bm{B} = \mathbf{P}\bm{\sigma}$,
$\bar{\mathbf{R}} = \diag^{-1}(\mathbf{a})
\mathbf{R}\diag^{-1}(\mathbf{a})$, and
$\mathbf{a}$ and $\mathbf{R}$ are, respectively, the $K\times 1$
expectation and $K\times K$ correlation matrix of the weight vector
$\bm{\theta}_m$.
\end{lemma} 

\section{A Geometric Perspective}
\vspace*{-2ex}
\label{sec:idea}
\begin{figure}[!htb]
\vglue -1ex
\centering{
\includegraphics[width=0.9\linewidth]{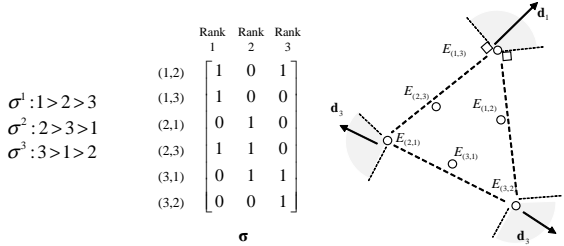}}
\vglue -1ex
\caption{\small{A separable ranking matrix $\bm{\sigma}$ with $K=3$
    rankings over $Q=3$ items, and the underlying geometry of the row
    vectors of $\mathbf{E}$. $(1,3),(2,1),(3,2)$ are novel
    pairs. Shaded regions depict the solid angles of the extreme
    points. }}
\label{fig:extreme}
\end{figure}
%
%
The key insight of our approach is an intriguing geometric property of
the normalized second-order moment matrix $\mathbf{E}$ (defined in
Lemma~\ref{lem:2ndOrder1}) illustrated in Fig.~\ref{fig:extreme}.
%
%
This arises from the so-called {\it separability condition} on the
ranking matrix $\bm{\sigma}$,
\begin{definition}
\label{separable_condition}
A ranking matrix $\bm{\sigma}$ is separable if for each ranking $k$,
there is at least one ordered pair $(i,j)$, such that
$\sigma_{(i,j),k}>0$ and $\sigma_{(i,j),l}=0$, $\forall ~ l\neq k$.
\end{definition}
In other words, for each ranking, there exists at least one ``novel''
pair of items $\{i,j\}$ such that $i$ is uniquely preferred over $j$
in that ranking while $j$ is ranked higher than $i$ in all the other
rankings.
%
Figure~\ref{fig:extreme} shows an example of a separable ranking
matrix in which the ordered pair $(1,3)$ is novel to ranking
$\bm{\sigma}^1$, the pair $(2,1)$ to $\bm{\sigma}^2$, and the pair
$(3,2)$ to $\bm{\sigma}^3$.

The separability condition has been identified as a good approximation
for real-world datasets in nonnegative matrix factorization
\citep{Donhunique:ref} and topic modeling \citep{Arora2:ref,
  Ding14:ref}, etc.
In the context of ranking, this condition has appeared, albeit
implicitly in a different form, in the seminal works of
\citep{Jagabathula08:ref, Farias09:ref}.
Moreover, as shown in \citep{Farias09:ref}, the separability condition
is satisfied with high probability when the $K \ll Q$ underlying
rankings are sampled uniformly from the set of all $Q!$ permutations.
In our experiments we have observed that the ranking matrix induced by
the rating matrix estimated by matrix factorization is often
separable (Sec.~\ref{sec:semi}).

%
If $\bm{\sigma}$ is separable then the novel pairs correspond to
extreme points of the convex hull formed by all the row vectors of
$\mathbf{E}$ (Fig.~\ref{fig:extreme}).
Thus, the novel pairs can be efficiently identified through an extreme
point finding algorithm.
Once all the novel pairs 
are identified, the ranking matrix can be estimated using a
constrained linear regression \citep{Arora2:ref, Ding14:ref}.
To exclude redundant rankings and ensure unique identifiability, we
assume $\mathbf{R}$ has full rank.

We leverage
the normalized {\bf Solid Angle} subtended by extreme points to detect
the novel pairs as proposed in \citep[][Definition 1]{Ding14:ref}.
The solid angles are indicated by the shaded regions in
Fig.~\ref{fig:extreme}.  From a statistical viewpoint, it can be
defined as the probability that a row vector $\mathbf{E}_{(i,j)}$ has
the maximum projection value along an isotropically distributed random
direction $\mathbf{d}$:
\begin{align}
\label{eq:solidangle}
\nonumber q_{(i,j)} \triangleq p \{ & \forall (s,t):
\mathbf{E}_{(i,j)}\neq \mathbf{E}_{(s,t)}, \\
& \qquad \langle \mathbf{E}_{(i,j)}, \mathbf{d} \rangle > \langle
\mathbf{E}_{(s,t)}, \mathbf{d} \rangle \}
\end{align}
%
These can be efficiently approximated using a few iid isotropic
$\mathbf{d}$'s. 
By following the approach in \citep[][Lemma 2]{Ding14:ref} for topic modeling, one
can prove the following result which shows that the solid angles can be used to detect novel pairs:
\begin{lemma}
\label{lem:solidangle}
Suppose $\bm{\sigma}$ is separable and $\mathbf{R}$ is full rank,
then, $q_{(i,j)} > 0$ if and only if $(i,j)$ is a novel pair.
\end{lemma}
%
This motivates the following solution approach:
$(1)$ Estimate the solid angles $q_{(i,j)}$,
$(2)$ Select $K$ distinct pairs with largest $q_{(i,j)}$'s,
and $(3)$ Estimate the ranking matrix $\bm{\sigma}$ using constrained
linear regression.


Given the estimated ranking matrix $\bm{\sigma}$ (and $\mathbf{B}$),
we follow the typical steps in topic modeling
\citep{Blei2012Review:ref} to fit the ranking prior, infer
user-specific preferences $\bm{\theta}_m$, and predict new comparisons
(see Sec.~\ref{sec:exp}).
%
\section{Algorithm and Analysis}
\label{sec:alg}
\vspace*{-1ex}
The main steps of our approach are outlined in Algorithm
~\ref{alg:highlevel}
and expanded in detail in Algorithms~\ref{alg:rp},~\ref{alg:esttopic}
and ~\ref{alg:post}.
Algorithm~\ref{alg:rp} detects all the novel pairs for the $K$
distinct rankings.
Once the novel pairs are identified, Algorithm~\ref{alg:esttopic}
estimates matrix $\mathbf{B}$ using constrained linear regression
followed by row and then column scaling.

Algorithm~\ref{alg:post} further processes $\widehat{\mathbf{B}}$ to
obtain an estimate of the ranking matrix $\bm{\sigma}$.
Step 1 is based on Eq.~\eqref{eq:sigma} and step 2 further rounds each
element to $0$ or $1$. Algorithm~\ref{alg:post} guarantees that
$\widehat{\bm{\sigma}}$ is binary and satisfies the condition:
$\widehat{\sigma}_{(i,j),k} + \widehat{\sigma}_{(j,i),k} =1$ for all
$i\neq j$ and all $k$.

\begin{algorithm}[!htb]
\caption{Ranking Recovery (Main Steps)}
\label{alg:highlevel}
\begin{algorithmic}[1]
\REQUIRE Pairwise comparisons $\widetilde{\mathbf X}$,
$\widetilde{\mathbf X}^{\prime} (W\times M)$; Number of rankings $K$;
Number of projections $P$; Tolerance parameters $\zeta,\epsilon >0$.
\ENSURE Ranking matrix estimate $\widehat{\bm{\sigma}}$.
\STATE Novel Pairs
$\mathcal{I}\leftarrow$NovelPairDetect($\widetilde{\mathbf
  X},\widetilde{\mathbf X}^{\prime}, K, \numofproj, \zeta$)
\STATE $\widehat{\mathbf{B}} \leftarrow$EstimateRankings($\mathcal{I},
\mathbf{X}, \epsilon$)
\STATE $\widehat{\bm{\sigma}} \leftarrow$PostProcess($\widehat{\mathbf{B}}$)
\end{algorithmic}
\end{algorithm}
%
%
\begin{algorithm}
\caption{NovelPairDetect (via Random Projections)}
\label{alg:rp}
\begin{algorithmic}
\REQUIRE $\widetilde{\mathbf X}$, $\widetilde{\mathbf X}^{\prime}$;
number of rankings $K$; number of projections $P$; tolerance $\zeta$;
\ENSURE $\mathcal{I}$: The set of all novel pairs of $K$ distinct rankings.
\STATE $\widehat{\mathbf{E}} \leftarrow M \widetilde{\mathbf
  X}^{\prime}\widetilde{\mathbf X}^{\top}$
\STATE $\forall (i,j)$, $\mathcal{J}_{(i,j)} \leftarrow \lbrace (s,t):
\widehat{E}_{(i,j),(i,j)}
-2\widehat{E}_{(i,j),(s,t)}+\widehat{E}_{(s,t),(s,t)} \geq \zeta/2
\rbrace$,
\FOR {$r=1,\ldots,P$}
\STATE Sample ${\mathbf d}_r \in\rR^{W}$ from an isotropic prior
	\STATE $\hat{q}_{(i,j),r} \leftarrow \mathbb{I}\lbrace \forall
        (s,t)\in\mathcal{J}_{(i,j)}, ~ \widehat{\mathbf{E}}_{(s,t)}
               {\mathbf d}_r \leq \widehat{\mathbf{E}}_{(i,j)}
               {\mathbf d}_r \rbrace$ , $\forall (i,j)$
\ENDFOR
\STATE $\hat{q}_{(i,j)} \leftarrow \frac{1}{P} \sum_{r = 1}^{{P}}
\hat{q}_{(i,j), r}$, $\forall (i,j)$
\STATE $k \leftarrow 0$,$l \leftarrow 1$, and $\mathcal{I}
\leftarrow\emptyset$
\WHILE {$k \leq K$}
\STATE $(s,t) \leftarrow$ index of the $l^{\text{th}}$ largest
value among $\hat{q}_{(i,j)}$'s
\IF {$(s,t) \in \bigcap_{(i,j)\in\mathcal{I}} \mathcal{J}_{(i,j)}$}
\STATE {$\mathcal{I} \leftarrow \mathcal{I} \cup \{(s,t)\}$, $~~k
\leftarrow k + 1$}
\ENDIF
\STATE $l \leftarrow l + 1$
\ENDWHILE
\end{algorithmic}
\end{algorithm}
%
%
\begin{algorithm}[!htb]
\caption{Estimate Rankings}
\label{alg:esttopic}
\begin{algorithmic}
\REQUIRE $\mathcal{I} = \{(i_1,j_1),\ldots, (i_K, j_K)\}$ the set of
novel pairs of $K$ rankings; ${\mathbf X}$, ${\mathbf X}^{\prime}$;
precision $\epsilon$
\ENSURE $\widehat{{\bm{B}}}$ as the estimate of ${\bm{B}}$. 
\STATE 
$
{\mathbf Y} = (\widetilde{\mathbf X}_{(i_1,j_1)}^{\top}, \ldots,
\widetilde{\mathbf X}_{(i_K,j_K)}^{\top})^{\top},$
\STATE
$
{\mathbf Y^{\prime}} = (\widetilde{\mathbf
  X}_{(i_1,j_1)}^{{\prime}\top}, \ldots, \widetilde{\mathbf
  X}_{(i_K,j_K)}^{{\prime}\top})^{\top}
$
\FORALL {$(i,j)$ pairs}
\STATE Solve $\widehat{\bm{\beta}}_{(i,j)} \leftarrow \argmin\limits_{\mathbf{b}} M
       (\widetilde{\mathbf X}_{(i,j)} - {\mathbf b} {\mathbf Y})
       (\widetilde{\mathbf X}^{\prime}_{(i,j)} - {\mathbf b} {\mathbf
         Y}^{\prime})^{\top} $
\STATE Subject to $b_k \geq 0, \sum_{k=1}^{K} b_k = 1$, With precision
$\epsilon$
\STATE $\widehat{\bm{\beta}}_{(i,j)} \leftarrow (\frac{1}{M} {\mathbf X}_{(i,j)} {\mathbf 1}) \widehat{\bm{\beta}}_{(i,j)}$
\ENDFOR
\STATE $\widehat{\bm{B}} \leftarrow$column normalize $\widehat{\bm{
    \beta}}$
\end{algorithmic}
\end{algorithm}
%
%
\begin{algorithm}
\caption{Post Processing}
\label{alg:post}
\begin{algorithmic}[1]
\REQUIRE $\widehat{\mathbf{B}}$ as the estimate of $\mathbf{B}$
\ENSURE $\widehat{\bm{\sigma}}$ as the estimate of $\bm{\sigma}$
\STATE $\widehat{{\sigma}}_{(i,j),k}\leftarrow
\frac{\widehat{{B}}_{(i,j),k}}{\widehat{{B}}_{(i,j),k} +
  \widehat{{B}}_{(j,i),k}}$, $\forall i,j\in\mathcal{U}, \forall k$
\STATE $\widehat{{\sigma}}_{(i,j),k} \leftarrow
\text{Round}[\widehat{{\sigma}}_{(i,j),k}]$, $\forall
i,j\in\mathcal{U}, \forall k$
\end{algorithmic}
\end{algorithm}

Our approach inherits the polynomial computational complexity of the
topic modeling algorithm in \cite{Ding14:ref}:
\begin{theorem}
\label{thm:computation}
The running time of Algorithm~\ref{alg:highlevel} is $\mathcal{O}(MNK
+ Q^{2}K^{3})$.
\end{theorem}
\vspace*{-1ex}

We further derive the sample complexity bounds for our approach which
is also polynomial in all model parameters and $\log(1/\delta)$ where
$\delta$ is the upper bound on error probability.
A major technical improvement compared to the results that appear in
\cite{Ding14:ref} is that our analysis holds true for {\it any
  isotropic distribution} over the random directions $\mathbf{d}$ in
Alg.~\ref{alg:rp}.
The previous result in \cite[][Theorem 1, 2]{Ding14:ref} was designed only for
specific distributions such as spherical Gaussian.
Formally, 
\begin{theorem}
\label{thm:sample}
Let the ranking matrix $\bm{\sigma}$ be separable and $\mathbf{R}$
have full rank. Then the Algorithm~\ref{alg:highlevel} can
consistently recover $\bm{\sigma}$ up to a column permutation as the
number of users $M\rightarrow \infty$ and number of projections
$P\rightarrow \infty$.
Furthermore, for any isotropically drawn random direction
$\mathbf{d}$, $\forall \delta>0$, if
\begin{align*}
M \geq \max \Biggl\{ 40 \frac{\log(3W/\delta)}{ N \rho^2 \eta^4}, ~
320 \frac{W^{0.5} \log(3W/\delta)}{N \eta^6 \lambda_{\min}} \Biggr\}
\end{align*}
and
$\numofproj \geq 16 \frac{\log(3W/\delta)}{q_{\wedge}^2}$, 
then Algorithm~1 fails with probability at most $\delta$.  The other
model parameters are defined as $\eta = \min_{1\leq w \leq W}
[\mathbf{B}\mathbf{a}]_w$, $\rho =\min\{ \frac{d}{8}, \frac{\pi d_2
  q_{\wedge}}{4 W^{1.5}} \}$, $d_2 \triangleq (1-b)\lambda_{\min}$,
$d=(1-b)^{2}\lambda_{\min}^{2}/\lambda_{\max}$,
$b=\max_{j\in\mathcal{C}_0,k} \bar{B}_{j,k} $ and $\lambda_{\min}$,
$\lambda_{\max}$ are the minimum /maximum eigenvalues of
$\bar{\mathbf{R}}$.  $q_{\wedge}$ is the minimum solid angle of the
extreme points of the convex hull of the rows of $\mathbf{E}$.
\end{theorem}
Detailed proofs are provided in the supplementary material. We combine
the analysis of Alg.~\ref{alg:post} and the re-scaling steps in
Alg.~\ref{alg:esttopic} in order to exploit the structural constraints
of the ranking model.  As a result, we obtain an improved sample
complexity bound for $M$ compared to \cite{Ding14:ref, Arora2:ref}
\vspace*{-1ex}
\section{Experimental Validation}
\label{sec:exp}
\vspace*{-1ex}
\subsection{Overview of Experiments and Methodology}
\vspace*{-2ex}
\label{sec:methodsandsetting}
We conduct experiments first on semi-synthetic dataset in order to
validate the performance of our proposed algorithm when the model
assumptions are satisfied, and then on real-world datasets in order to
demonstrate that the proposed model can indeed effectively capture the
variability that one encounters in the real world.
We focus on the collaborative filtering applications where population
heterogeneity and user inconsistency are the well-known
characteristics \citep[e.g.,][]{BPMF:ref}.
We use Movielens, a benchmark movie-rating dataset widely used in the
literature.\footnote{Another large benchmark, Netflix dataset is not
  available due to privacy issues. Movielens is currently available at
  \url{http://grouplens.org/datasets/movielens/}}
The rating-based data is selected due to its public availability and
widespread use, but we convert it to pairwise comparisons data and
focus on modeling from a ranking viewpoint. This procedure has been
suggested and widely used in the rank-\-aggregation literature
\citep[e.g.,][]{Lu11:ref,volkovs14a:ref}.
For the semi-synthetic datasets, we evaluate the {\bf reconstruction
  error} between the learned rankings $\widehat{\bm{\sigma}}$ and the
ground truth. We adopt the standard {\it Kendall's tau distance}
between two rankings.
For the real-world datasets where true parameters are not available,
we use the {\bf held-out log-likelihood}, a standard metric in ranking
prediction \citep{Lu11:ref} and in topic modeling
\cite{Wallach09:ref}.

In addition, we consider the standard task of rating prediction via
our proposed ranking model. Our aim here is to illustrate that our
model is suitable for real-word data. We do not optimize tuning
parameters in order to achieve the best result.
We measure the performance by {\bf root-mean-square-error} (RMSE)
which is the standard in literature\citep[e.g.,][]{BPMF:ref,
  Netflix:ref}.

The parameters of our algorithm are the same as in
\cite{Ding14:ref}. Specifically, the number of random projections $P =
150\times K$, the tolerance parameter $\zeta/2$ for Alg.~\ref{alg:rp}
is fixed at $0.01$ and the precision parameter $\epsilon = 10^{-4}$
for Alg.~\ref{alg:esttopic}.
\vspace*{-2ex}
\subsection{Semi-synthetic simulation}
\vspace*{-2ex}
\label{sec:semi}
We first use a {\it semi-synthetic} dataset to validate the
performance of our algorithm.
In order to match the dimensionality and other characteristics that
are representative of real-world examples,
we generate the semi-synthetic pairwise comparisons dataset using a
benchmark movie star-ratings dataset, Movielens.
The original dataset has approximately $1$ million ratings for $3952$
movies from $M=6040$ users. The ratings range from 1 star to 5 stars.
\begin{figure}[!htb]
\includegraphics[width=0.90\linewidth]{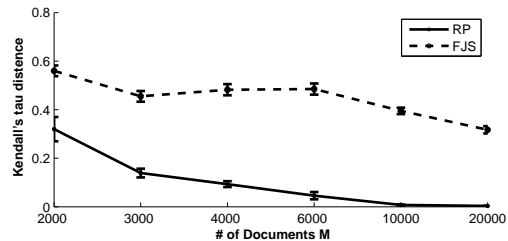}
\caption{{\small The normalized Kendall's tau distance error of the
    estimated rankings, as functions of $M$, estimated by RP and FJS
    from the semi-synthetic dataset with $Q=100, N=300, K=10$.}}
\label{fig:synthetic}
\vspace*{-2ex}
\end{figure}

We follow the procedure in \citep{Lu11:ref} and \citep{volkovs14a:ref}
to generate the semi-synthetic dataset as follows. We consider the
$Q=100$ most frequently rated movies and train a latent factor model
on the star-ratings data using a state-of-the-art matrix factorization
based algorithm \citep{BPMF:ref}.
This approach is selected for its state-of-the-art performance on many
real-world collaborative filtering tasks.
This procedure learns a $Q\times K$ movie-factor matrix whose columns
are interpreted as scores of the $Q$ movies over the $K$ latent
factors\citep{BPMF:ref,volkovs14a:ref}. By sorting the scores of each
column of the movie-factor matrix, we obtain $K$ rankings for
generating the semi-synthetic dataset. We set $K=10$ as suggested by
\cite{Lu11:ref} and \cite{BPMF:ref}. We note that the resulting
ranking matrix $\bm{\sigma}$ satisfies the separability condition.

The other model parameters are set as follows. $\mu_{i,j} =
1/{Q\choose 2}$, $\forall i,j \in\mathcal{U}$.
The prior distribution for $\bm{\theta}_m$ is set to be Dirichlet
$\Pr(\bm{\theta}_m \vert \bm{\alpha})
=\frac{1}{C}\prod\limits_{k=1}^{K} \theta_{k}^{\alpha_k -1}$ as
suggested by \citep{Lu11:ref}.
The parameters $\alpha_k$'s are determined by $\alpha_k = \alpha_0
a_k$, where the concentration parameter $\alpha_0 = 0.1$ and the
expectation $\mathbf{a}=[a_1,\ldots, a_K]^{\top}$ is sampled uniformly
from the $K=10$ dimensional simplex for each random realization.
We note that the correlation matrix $\mathbf{R}$ of the Dirichlet
distribution has full rank \citep{Arora2:ref}.
We fix $N = 300$ comparisons per user to approximate the observed
average pairwise comparisons in the Movielens dataset and vary $M$.

Since the output of our algorithm is determined only up to a column
permutation, we first align the columns of $\bm{\sigma}$ and
$\widehat{\bm{\sigma}}$ using bipartite matching based on $\ell_1$
distance, and then measure the performance by the $\ell_1$ distance
between the ground truth rankings $\bm{\sigma}$ and the estimate
$\widehat{\bm{\sigma}}$.
Due to the way $\bm{\sigma}$ is defined, this is equivalent to the
widely-used {\it Kendall's tau distance} between two rankings which is
proportional to the number of pairs in which two ranking schemes
differ.
We further normalize the $\ell_1$ error by $W=Q\times(Q-1)$ so that
the error measure for each column is a number between $[0,1]$.

We compare our proposed algorithm (denoted by RP) against the
algorithm proposed in \citep{Jagabathula08:ref,Farias09:ref} (denoted
by FJS) for estimating the ranking matrix. To the best of our
knowledge, this is the most recent algorithm with consistency
guarantees for $K>1$.\footnote{We show in the appendix that Alg.~FJS
  can be applied to our generative scheme since it only uses the first
  order statistics, and all the technical conditions are satisfied.}
We compared how the estimation error varies with the number of users
$M$, and the results are depicted in Fig.~\ref{fig:synthetic}.
For each setting, we average over $10$ Monte Carlo runs. Evidently,
our algorithm shows superior performance over FJS. More specifically,
since our ground truth ranking matrix is separable, as $M$ increases,
the estimation error of RP converges to zero, and the convergence is
much faster than FJS. We note that only for $M \geq 100,000$ does the
error of the FJS algorithm eventually start approaching 0.
\vspace*{-2ex}
\subsection{Movielens - Comparison prediction}
\vspace*{-2ex}
\label{sec:movie}
We apply the proposed algorithm (RP) to the real-world Movielens
dataset introduced in Sec.~\ref{sec:semi} and consider the task of
predicting pairwise comparisons.
We consider two settings: $(1)$ new comparison prediction, and $(2)$
new user prediction.
We train and evaluate our model using the comparisons 
obtained from the star-ratings of the Movielens dataset. This
procedure of generating comparisons from star-ratings is motivated by
\citep{Lu11:ref,volkovs14a:ref}.
We focus on the $Q=100$ most frequently rated movies and obtain a
subset of $183,000$ star-ratings from $M=5940$ users. The pairwise
comparisons are generated from the star ratings following
\citep{Lu11:ref, volkovs14a:ref}: for each user $m$, we {\bf select}
pairs of movies $i,j$ that user $m$ rated, and {\bf compare} the stars
of the two movies to generate comparisons.

To {\bf select} pairs of items to compare, we consider: $(a)$ (Full)
all pairs of movies that a user has rated, or $(b)$ (Partial) randomly
select $5 N_{star,m}$ pairs where $N_{start,m}$ is the number of
movies user $m$ has rated.

To {\bf compare} a pair of movies $i,j$ rated by a user, $w_{m,n} =
(i,j)$ if the star rating of $i$ is higher than $j$. For ties, we
consider: $(i)$(Both) generate $w_{m,1}=(i,j)$ and $w_{m,2}=(j,i)$,
$(ii)$ (Ignore) do nothing, and $(iii)$ (Random) select one of
$w_{m,1} ,w_{m,2}$ with equal probability.

{\bf New comparison prediction: } In this setting, for each user, a
subset of her ratings are used to generate the training comparisons
while the remaining are for testing comparisons.
We follow the training/testing split as in
\citep{BPMF:ref}.\footnote{The training/testing split is available at
  \url{http://www.cs.toronto.edu/~rsalakhu/BPMF.html}}
%
%
We convert both the training ratings and testing ratings into training
comparisons and testing comparisons independently.
We evaluate the performance by the predictive log-likelihood of the
testing data, i.e., $\Pr(\mathbf{w}_{test}\vert
\mathbf{w}_{train},\widehat{\bm{\sigma}})$.
Given the estimate $\widehat{\bm{\sigma}}$, we follow
\citep{Arora2:ref,Ding14:ref} to fit a Dirichlet prior model.
We then calculate the prediction log-likelihood using the
approximation in \citep{Wallach09:ref} which is now the standard.
%
We compare against the FJS algorithm. Figure~\ref{fig:mv1}(upper)
summarizes the results for different strategies in generating the
pairwise comparisons with $K=10$ held fixed.
The log-likelihood is normalized by the total number of pairwise
comparisons tested.
As depicted in Fig.~\ref{fig:mv1} (upper), the log-likelihood produced
by the proposed algorithm RP is higher, by a large margin, compared to
FJS. The predictive accuracy is robust to how the comparison data is
constructed.
We also consider the normalized log-likelihood as function of $K$ (see
Fig.~\ref{fig:mv3}). The results validate the superior performance and
suggest that $K=10$ is a reasonable parameter choice.
\begin{figure}[!htb]
\centering
\begin{minipage}[b]{.90\linewidth}
  \centering
  \centerline{\includegraphics[width=0.9\linewidth]{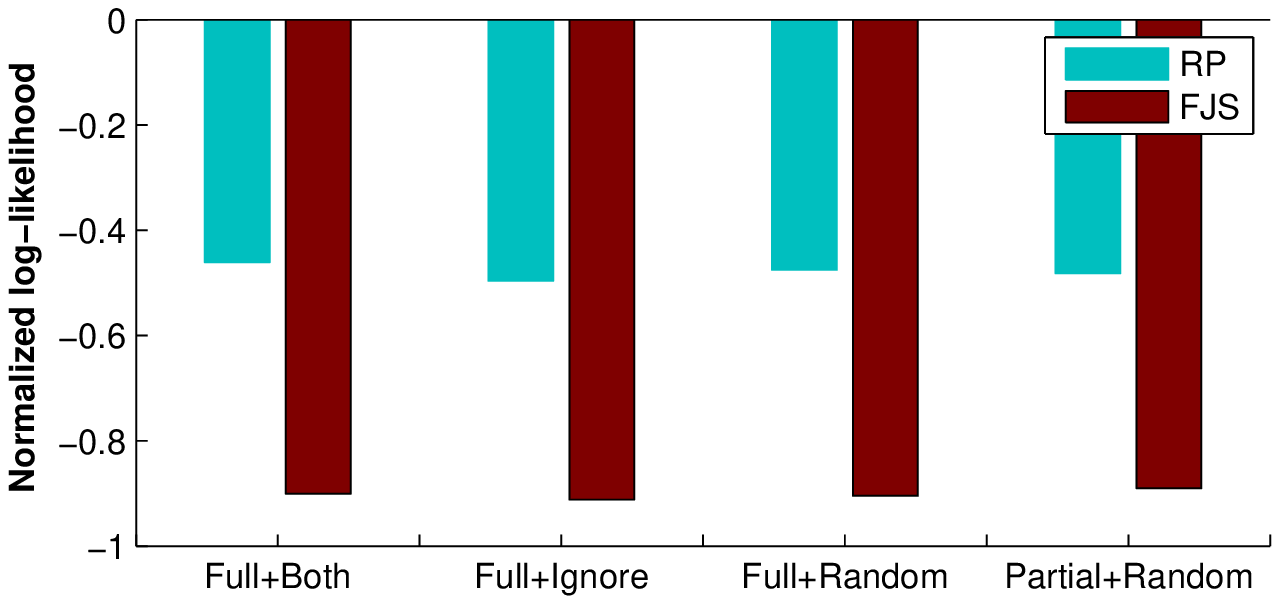}}
\end{minipage}
\begin{minipage}[b]{0.90\linewidth}
  \centering
  \centerline{\includegraphics[width=0.9\linewidth]{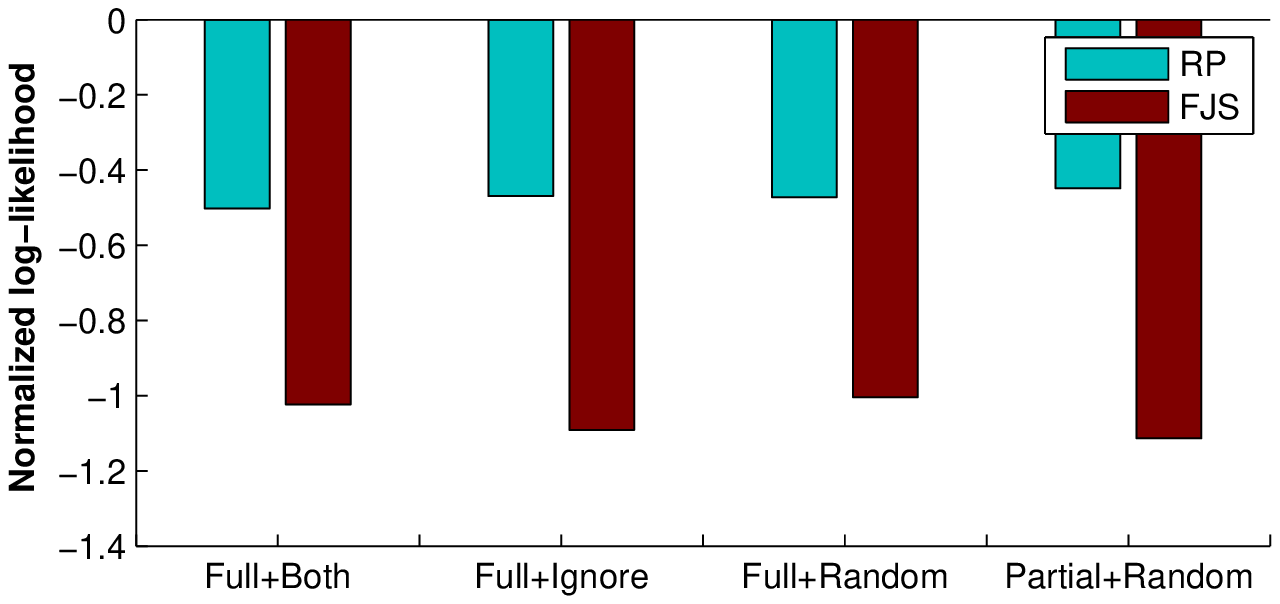}}
\end{minipage}
\caption{{\small The normalized log-likelihood under different
    settings for (upper) new comparison prediction and (lower) new
    user prediction on the truncated Movielens.  $K=10$.}}
\vspace*{-2ex}
\label{fig:mv1}
\end{figure}
\begin{figure}[!htb]
\includegraphics[width=0.90\linewidth]{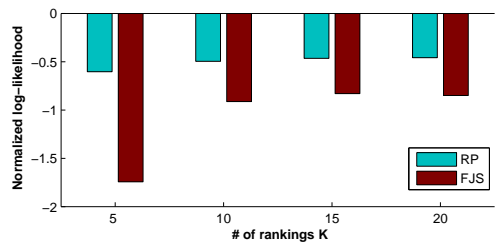}
\vspace*{-1ex}
\caption{{\small The normalized log-likelihood for Full + Ignore
    strategy for various $K$ on the truncated Movielens dataset (new
    comparison prediction).}}
\label{fig:mv3}
\vspace*{-2ex}
\end{figure}
%

%
{\bf New user prediction: } In this setting, all the ratings of a
subset of users are used to generate the training comparisons while
the remaining users' comparisons are used for testing.
Following \citep{Lu11:ref}, we split the first $4000$ users (in the
original dataset) in the Movielens dataset for training, and the
remaining for testing.
%
We use the held-out log-likelihood, i.e., $\Pr(\mathbf{w}_{test}\vert
\widehat{\bm{\sigma}})$ to measure the performance. The
log-likelihoods are again calculated using the standard Gibbs Sampling
approximation \citep{Wallach09:ref}. We compare our algorithm RP with
the FJS algorithm. The log-likelihoods are then normalized by the
total number of comparisons in the testing phase.  We fix the number
of rankings at $K=10$. The results which are summarized in
Fig.~\ref{fig:mv1} (lower) agree with the results of the previous
task.
\vspace*{-2ex}
\subsection{Movielens - Rating prediction via ranking model}
\vspace*{-2ex}
\label{sec:rating}
The purpose of this experiment is to {\bf illustrate} that our ranking model
can capture real-world user behavior through rating predictions, one
important task in personal recommendation \citep{Netflix:ref}.
We first train our ranking model using the training comparisons, and
then predict ratings based on comparison prediction. Our objective is
to demonstrate results comparable to the state-of-the-art
rating-\-based methods rather than achieving the best possible
performance on certain datasets.

We use the same training/testing rating split from \citep{BPMF:ref} as
used in {\it new comparison prediction} in Sec.~\ref{sec:movie}, and
focus only on the $Q=100$ most rated movies. We first convert the
training ratings into training comparisons (for each user, all pairs
of movies she rated in the training set are converted into comparisons
based on the stars and the ties are ignored) and train a ranking
model. The prior is set to be Dirichlet.

To predict stars from comparison prediction, we propose the following
method.
Consider the problem of predicting $r_{i,m}$, i.e., the rating of user
$m$ on movie $i$. We assume $r_{i,m} = s, s=1,\ldots, 5$, then compare
it against the ratings on movie $\{j_1,\ldots, j_V\}$ she has rated in
training. This generates a set of pairwise comparisons
$\mathbf{w}_{i,m}(s)$.
For example, if user $m$ has rated movies $A,B,C$ with stars $4,2,5$
respectively in the training set and we are predicting her rating $s$
of movie $D$. Then for $s=3$, $\mathbf{w}_{D,m}(3) = \{ (A,D), (D,B),
(C,D) \}$ while for $s=1$, $\mathbf{w}_{D,m}(1) = \{ (A,D), (B,D),
(C,D) \}$.
We then chose $s$ to maximize the likelihood of $\mathbf{w}_{i,m}(s)$,
\begin{equation*}
\hat{r}_{i,m} = \arg\max_{s} p(\mathbf{w}_{i,m}(s) \vert \mathbf{w}_{train},\widehat{\bm{\sigma}}).
\vspace*{-1ex}
\end{equation*}
We evaluate the performance using root-mean-square-error (RMSE). This
is a standard metric in collaborative filtering \citep{Netflix:ref}.
\footnote{
Normalized Discounted Cumulative Gain (nDCG) is another standard
metric. It requires, however, to predict a total ranking and is
inapplicable in our test setting.  }
We compared our ranking-based algorithm, RP , against rating based algorithms. We choose to compare two benchmark
algorithms, Probability Matrix Factorization (PMF) in \citep{PMF:ref},
and Bayesian probability matrix factorization (BPMF) in
\citep{BPMF:ref} for their robust empirical performance 
\footnote{The implementation is available at
  \url{http://www.cs.toronto.edu/~rsalakhu/BPMF.html}}.
Both PMF and BPMF are latent factor models. The number of latent
factors $K$ has the similar interpretation as in our ranking
model. The RMSE for different choices of $K$ are summarized in
Table~\ref{table:rmse}.
\begin{table}[!tbh]
\vspace*{-2ex}
\caption{Testing RMSE on the Movielens dataset}
\label{table:rmse}
\centering{{\small
\begin{tabular}{|c|c|c|c|c|}
\hline 
$K$ & PMF & BPMF & RP & BPMF-int \\ 
\hline 
10 & 1.0491 & 0.8254 & 0.8840 & 0.8723 \\ 
\hline 
15 & 0.9127 & 0.8236 & 0.8780 & 0.8734 \\ 
\hline 
20 & 0.9250 & 0.8213 & 0.8721 & 0.8678 \\ 
\hline 
\end{tabular}
} }
\vglue -3ex
\end{table}

Although coming from a different feature space and modeling
perspective, our approach has similar RMSE performance as the
rating-based PMF and BPMF.
Since the ratings predicted by our algorithm are integers from $1$ to
$5$, we also consider restricting the output of BPMF to be integers
(denote as BPMF-int). This is achieved by rounding the real-valued
prediction of BPMF to the nearest integer from 1 to 5.
We observe that our RP algorithm outperforms PMF which is known for
over-fitting issues, and matches the performance of BPMF-int.
This demonstrates that our approach is in fact suitable for modeling
real-world user behavior.

We point out that one can potentially improve these results by
designing a better comparison generating strategy, ranking prior,
aggregation strategies, etc. This is, however, beyond the scope of
this paper.

We note that our proposed algorithm can be naturally parallelized in a distributed database for web scale problems as demonstrated in \citep{Ding14:ref}. The statistical efficiency of the centralized version can be retained with an insignificant communication cost. 

\vspace*{-2ex}
\section*{Acknowledgment}
\vspace*{-2ex} 
This article is based upon work supported by the
U.S. AFOSR and the U.S. NSF under award numbers \# FA9550-10-1-0458
(subaward \# A1795) and \# 1218992 respectively. The views and
conclusions contained in this article are those of the authors and
should not be interpreted as necessarily representing the official
policies, either expressed or implied, of the agencies.
%





\appendix
\section*{SUPPLEMENTARY MATERIAL}
While our analysis of the proposed approach and algorithm largely
tracks the methodology in \citep{Ding14:ref}, here we develop a set of
new analysis tools that can handle more general settings.
Specifically, our new analysis tools can handle {\bf any}
isotropically distributed random projection directions. In contrast,
the work in \citep[e.g.,][]{Ding14:ref} can only handle special types
of random projections, e.g., spherical Gaussian. 
Our new refined analysis can not only handle more general settings, it
also gives an overall improved sample complexity bound.

We also analyse the post-processing step in Algorithm~4. This step
accounts for the special constraints that a valid ranking
representations must satisfy and guarantees a binary-valued estimate
of $\bm{\sigma}$. It should also satisfy the property that either $\sigma^{k}(i) > \sigma^{k}(j) $ or $\sigma^{k}(i) < \sigma^{k}(j)$ for all distinct
$i,j$ and all $k$.

We note that the analysis framework that we present here for the solid
angle can in fact be extended to handle other types distributions for
the random projection directions. This is, however, beyond the scope
this paper.

\section{On the generative model}

\begin{proposition}
$\mathbf{B} = \mathbf{P}\bm{\sigma}$ is column stochastic.
\end{proposition}

\begin{proof}
Noting that $\sigma_{(i,j),k} + \sigma_{(j,i),k} = 1$ by definition, and $P_{(i,j),(i,j)} = P_{(j,i),(j,i)}=\mu_{i,j}$, therefore,
\begin{eqnarray*}
\sum_{(i,j)} B_{(i,j),k} &=& \sum_{(i,j)~:~i<j} (\sigma_{(i,j),k}+\sigma_{(j,i),k}) \mu_{i,j} \\
&=& \sum_{(i,j)~:~i<j}\mu_{i,j} = 1
\end{eqnarray*}
\end{proof}

\section{Connection to the model in FJS}
Here we discuss in detail the connection to the probability model as well as the algorithm proposed in \cite{Jagabathula08:ref}\cite{Farias09:ref} (denoted by FJS). 

First, the generative model proposed in FJS can be viewed as a special
case of our generative model. If we consider the prior distribution of
$\bm{\theta}_m$ to be a pmf on the vertices of the $K$-dimensional
probability simplex (so that $\bm{\theta}_m$ has only one nonzero
component with probability one), i.e., 
\begin{equation}
\Pr(\bm{\theta}_m =\mathbf{e}_k) = b_k
\end{equation}
where $\mathbf{e}_k$ is the $k$-th standard basis vector and
$\sum_{k=1}^{K} b_k = 1$, then each user $m$ is associated with only
one of the $K$ types with probability $b_k$ for the $k$-th type.
We note that under this prior, $\mathbf{a} \triangleq
\eE(\bm{\theta}_m) = \mathbf{b}$ and $ \mathbf{R}\triangleq
\eE(\bm{\theta}_m \bm{\theta}_m^{\top}) = \diag(\mathbf{b})$ has full
rank.

Second, the algorithm proposed in FJS can certainly be applied to our
more general setting. Since the algorithm FJS only uses the first
order statistic which corresponds to pooling the comparisons from all
the users together, it suffices to consider only the probabilities of
$p(w_1 = (i,j))$ by marginalizing over $\bm{\theta}$:
\begin{eqnarray*}
p(w_1 = (i,j)) & = & \int_{\bm{\theta}_m} p(w_1 = (i,j)\vert \bm{\theta}_m) \Pr(\bm{\theta}_m)  d\bm{\theta_m} \\
& = & \sum_{k=1}^{K} \sigma_{(i,j), k} \int_{\bm{\theta}_m} {\theta}_{k,m} d\bm{\theta_m} \\
& = & \sum_{k=1}^{K} \sigma_{(i,j), k} a_{k} \\
& = & \sum_{k ~: \sigma^{k}(i) < \sigma^{k}(j)} a_{k},
\end{eqnarray*}
where the last step is due to the definition of the ranking matrix
$\bm{\sigma}$.
The above derivation shows that if the expectation vector in our
generative model equals that in the model of FJS, then the probability
distribution of the first order statistic in both models will be
identical and the two models will be indistinguishable in terms of the
first order statistic. This shows that the comparison with FJS in the
experiments conducted in Sections 6.1 and 6.2 of the main paper is
both sensible and fair. \\
 
\noindent{\bf Indexing convention:} For convenience, for the rest of
this appendix we will index the $W=Q(Q-1)$ rows of $\mathbf{B}$ and
$\mathbf{E}$ by just a single index $i$ instead of an ordered pair
$(i,j)$ as in the main paper.

\section{Proof of Lemma 2 in the main paper}
Lemma 2 in the main paper is a result about the almost sure convergence of the estimate
of the normalized second order moments $\mathbf{E}$. Our proof of this
result will also provide an attainable rate of convergence. 

We first provide a generic method to establish the convergence rate for a function $\psi(\mathbf{X})$ of $d$ random variables $X_1,\ldots, X_d$ given their individual convergence rates. 
\begin{proposition}
\label{prop:convergence}
Let $\mathbf{X} = \left[X_1,\ldots, X_d \right]$ be $d$ random variables and $\mathbf{a}=\left[a_1,\ldots, a_d \right]$ be positive constants. 
Let $\mathcal{E}:=\bigcup\limits_{i \in\mathcal{I}} \lbrace \vert X_i - a_i\vert \geq \delta_i \rbrace$ for some constants $ \delta_i >0$,
and $\psi(\mathbf{X})$ be a continuously differentiable function in
$\mathcal{C} := \mathcal{E}^{c} $.
If for $i=1,\ldots, d$, $\Pr(\vert X_i - a_i \vert \geq \epsilon) \leq f_i(\epsilon)$ and $\max\limits_{X\in\mathcal{C}}\vert \partial_i\psi(\mathbf{X})\vert \leq C_i$, then, 
\begin{equation*}
\Pr(\vert \psi(\mathbf{X}) - \psi(\mathbf{a})\vert \geq \epsilon)  \leq \sum\limits_{i}f_{i}(\gamma ) + \sum\limits_{i=1} f_{i}(\frac{\epsilon}{d C_i})
\end{equation*}
\end{proposition}
\begin{proof}
Since $\psi(\mathbf{X})$ is continuously differentiable in
$\mathcal{C}$, $\forall \mathbf{X}\in\mathbf{C}, \exists \lambda \in
(0,1)$ such that
\begin{equation*}
\psi(\mathbf{X}) - \psi(\mathbf{a}) = \nabla^{\top}\psi((1-\lambda)\mathbf{a} + \lambda\mathbf{X})\cdot (\mathbf{X}-\mathbf{a})
\end{equation*}
Therefore, 
\begin{align*}
&\Pr(\vert \psi(\mathbf{X}) - \psi(\mathbf{a})\vert \geq \epsilon) \\
\leq & \Pr(\mathbf{X}\in\mathcal{E})  \\
 & + \Pr(\sum\limits_{i=1}^{d}\vert \partial_{i}\psi((1-\lambda)\mathbf{a} + \lambda\mathbf{X})	\vert \vert X_i - a_i \vert\geq \epsilon \vert \mathbf{X}\in\mathcal{C}	)\\
\leq & \sum\limits_{i\in\mathcal{I}} \Pr(\vert X_i - a_i\vert \geq \delta_i) \\
& + \sum\limits_{i=1}^{d}\Pr(\max\limits_{\mathbf{x}\in\mathcal{C}}\vert \partial_i \psi(\mathbf{x}) \vert\vert X_i - a_i\vert \geq \epsilon/d) \\ 
= & \sum\limits_{i\in\mathcal{I}}f_{i}(\delta_i ) + \sum\limits_{i=1} f_{i}(\frac{\epsilon}{d C_i})
\end{align*}
\end{proof}

Now we are ready to prove Lemma 2 of the main paper. Recall that
$\widetilde{\mathbf{X}}$ and $\widetilde{\mathbf{X}}^{\prime}$ are
obtained from $\mathbf{X}$ by first splitting each user's comparisons
into two independent copies and then re-scaling the rows to make them
row-stochastic. Therefore,
$\widetilde{\mathbf{X}} = \diag^{-1}(\mathbf{X}\mathbf{1})\mathbf{X}$. 
Since $\bar{\mathbf{B}} = \diag^{-1}(\mathbf{B} \mathbf{a}) \mathbf{B}
\diag(\mathbf{a})$, $\bar{\mathbf{R}}
=\diag^{-1}(\mathbf{a})\mathbf{R}\diag^{-1}(\mathbf{a})$, and $
\bar{\mathbf{B}}$ is row stochastic. From Lemma 2 of the main paper,
we have
\begin{lemma}
\label{lem:converge}
Let $\widehat{\mathbf{E}} = M \widetilde{\mathbf{X}^{\prime}}
\widetilde{\mathbf{X}^{\top}}$ and $\mathbf{E} =
\bar{\mathbf{B}}\bar{\mathbf{R}} \bar{\mathbf{B}}^{\top}$. If $\eta =
\min_{1\leq i\leq W}(\mathbf{B}\mathbf{a})_i >0$, then,
\begin{equation}
\Pr(\Vert \widehat{\mathbf{E}} - \mathbf{E}\Vert_{\infty} \geq \epsilon) \leq 8W^2\exp(-\epsilon^2 \eta^4 MN /20)
\end{equation}
\end{lemma}
\begin{proof}
For any $1\leq i,j\leq W$,
\begin{align*}
\widehat{E}_{i,j} & = M \frac{1}{\sum\limits_{m=1}^{M} X_{i,m}^{\prime}} (\sum\limits_{m=1}^{M} X_{i,m}^{\prime} X_{j,m} ) \frac{1}{\sum\limits_{m=1}^{M} X_{i,m}} \\
& = \frac{1/M \sum\limits_{m=1}^{M}(X_{i,m}^{\prime} X_{j,m})}{(1/M \sum\limits_{m=1}^{M}X_{i,m}^{\prime} )(1/M \sum\limits_{m=1}^{M} X_{j,m})} \\
& = \frac{\frac{1}{MN^2} \sum\limits_{m=1,n=1,n'=1}^{M,N,N} \mathbb{I}(w_{m,n}=i) \mathbb{I}(w_{m,n'}^{\prime}=j)}{\frac{1}{MN} \sum\limits_{m=1, n=1}^{M,N} \mathbb{I}(w_{m,n} =i)  \frac{1}{MN} \sum\limits_{m=1, n=1}^{M,N} \mathbb{I}(w_{m,n}^{\prime} =i)} \\
& : = \frac{F_{i,j}(M,N)}{G_{i}(M,N) H_{j}(M,N)}
\end{align*}

From the Strong Law of Large Numbers and equations (1), (2) in the
main paper, we have
\begin{align*}
& F_{i,j}(M,N) \xrightarrow{a.s.} \eE (\mathbb{I}(w_{m,n}=i) \mathbb{I}(w_{m,n'}^{\prime}=j))\\
 & \qquad = (\mathbf{B}\mathbf{R}\mathbf{B}^{\top})_{i,j} :=p_{i,j} 
\end{align*}
\begin{align*}
& G_{i}(M,N) \xrightarrow{a.s.} \eE (\mathbb{I}(w_{m,n}^{\prime}=i) ) = (\mathbf{B}\mathbf{a})_{i} :=p_{i} \\
& H_{i}(M,N) \xrightarrow{a.s.} \eE (\mathbb{I}(w_{m,n}=j) ) = (\mathbf{B}\mathbf{a})_{j} :=p_{j}
\end{align*}
and $\frac{(\mathbf{B}\mathbf{R}\mathbf{B}^{\top})_{i,j}}{(\mathbf{B}
  \mathbf{a})_{i} (\mathbf{B} \mathbf{a})_{j}} = \mathbf{E}_{i,j}$ by
definition.
Using McDiarmid's inequality, we obtain
\begin{align*}
& \Pr(\vert F_{i,j} - p_{i,j} \vert \geq \epsilon )\leq 2\exp(-\epsilon^2MN) \\ 
& \Pr(\vert G_{i} - p_{i} \vert \geq \epsilon )\leq 2\exp(-2\epsilon^2MN) \\
& \Pr(\vert H_{j} - p_{j} \vert \geq \epsilon )\leq 2\exp(-2\epsilon^2MN) 
\end{align*}
In order to calculate $\Pr\lbrace \vert \frac{F_{i,j}}{G_{i} H_{j}} -
\frac{p_{i,j}}{p_{i} p_{j}} \vert \geq \epsilon \rbrace$, we apply the
results from Proposition~\ref{prop:convergence}.
Let $\psi(x_1, x_2, x_3) = \frac{x_1}{x_2 x_3}$ with $x_1,x_2,x_3 >0$,
and $a_1 = p_{i,j}$, $a_2 = p_i$, $a_3= p_j$.
Let $\mathcal{I}=\{2,3\}$, $\delta_2 = \gamma p_i$, and $\delta_3 =
\gamma p_j$.
Then $\vert \partial_1\psi \vert = \frac{1}{x_2 x_3} $, $\vert
\partial_2\psi \vert = \frac{x_1}{x_2^{2} x_3} $, and $\vert
\partial_3\psi \vert = \frac{x_1}{x_2 x_3^{2}} $.

If $F_{i,j} = x_1$, $G_i = x_2$, and $H_j = x_3$, then $F_{i,j} \leq
G_i$, $F_{i,j}\leq H_j$. Then note that
\begin{align*}
C_1 &= \max_{\mathcal{C}} \vert \partial_1\psi \vert =  \max_{\mathcal{C}}  \frac{1}{G_i H_j} \leq \frac{1}{(1-\gamma)^{2}p_i p_j} \\
C_2 &= \max_{\mathcal{C}} \vert \partial_2\psi \vert =  \max_{\mathcal{C}}  \frac{F_{i,j}}{G_i^{2} H_j} \leq   \max_{\mathcal{C}}  \frac{1}{G_i H_j} \leq \frac{1}{(1-\gamma)^{2}p_i p_j} \\
C_3 &= \max_{\mathcal{C}} \vert \partial_3\psi \vert =  \max_{\mathcal{C}}  \frac{F_{i,j}}{G_i H_j^{2}} \leq   \max_{\mathcal{C}}  \frac{1}{G_i H_j}\leq \frac{1}{(1-\gamma)^{2}p_i p_j}
\end{align*}
By applying Proposition~\ref{prop:convergence}, we get
\begin{align*}
& \Pr\lbrace \vert \frac{F_{i,j}}{G_{i} H_{j}} - \frac{p_{i,j}}{p_{i} p_{j}} \vert \geq \epsilon \rbrace \\
\leq & \exp(-2\gamma^{2}p_i^{2} MN) + \exp(-2\gamma^{2}p_j^{2} MN) \\
& + 2\exp (-\epsilon^{2} (1-\gamma)^{4}(p_i p_j)^{2} MN/9)  \\
 & + 4 \exp (-2\epsilon^{2} (1-\gamma)^{4}(p_i p_j)^{2} MN/9) \\ 
\leq & 2\exp(-2\gamma^{2}\eta^{2}MN) + 6\exp(-\epsilon^{2} (1-\gamma)^{4}\eta^{4} MN/9)
\end{align*}
where $\eta = \min_{1\leq i \leq W} p_i$. 
There are many strategies for optimizing the free parameter
$\gamma$. We set $2\gamma^{2} = \frac{(1-\gamma)^{4}}{9}$ and solve
for $\gamma$ to obtain
\begin{align*}
 \Pr\lbrace \vert \frac{F_{i,j}}{G_{i} H_{j}} - \frac{p_{i,j}}{p_{i} p_{j}} \vert \geq \epsilon \rbrace \leq  8\exp(-\epsilon^{2}\eta^{4} MN/20)
\end{align*}
Finally, by applying the union bound to the $W^2$ entries in
$\widehat{\mathbf{E}}$, we obtain the claimed result.
\end{proof}
%
\section{Proof of Theorem 2 in the main paper}
\label{sec:thm2}
\subsection{Outline}
We focus on the case when the random projection directions are sampled
from {\bf any} {\it isotropic distribution}. Our proof is not tied to
the special form of the distribution; just its isotropic nature.
In contrast, the method in \citep[e.g.,][]{Ding14:ref} can only handle
special types of distributions such as the spherical Gaussian.

The proof of Theorem 2 in the main paper can be decoupled into two steps. First, we show
that Algorithm 2 in the main paper can consistently identify all the
novel words of the $K$ distinct rankings.  Then, given the success of
the first step, we will show that Algorithm 3 proposed in the main
paper can consistently estimate the ranking matrix $\bm{\sigma}$.
\subsection{Useful propositions}
We denote by $\mathcal{C}_k$ the set of all novel pairs of the ranking
$\sigma^{k}$, for $k=1,\ldots, K$, and denote by $\mathcal{C}_0$ the
set of other non-novel pairs. We first prove the following result.
\begin{proposition}
\label{prop:similarity}
Let $\mathbf{E}_i$ be the $i$-th row of $\mathbf{E}$. Suppose
$\bm{\sigma}$ is separable and $\mathbf{R}$ has full rank, then the
following is true:
\begin{table}[H]
\centering
\begin{tabular}{|c|c|c|}
\hline 
 & $\Vert \mathbf{E}_i - \mathbf{E}_j \Vert$ & $E_{i,i}-2E_{i,j}+E_{j,j}$ \\ 
\hline 
$i\in\mathcal{C}_1, j \in\mathcal{C}_1$ & $0$ & $0$ \\ 
\hline 
$i\in\mathcal{C}_1, j \notin\mathcal{C}_1$ & $\geq (1-b)\lambda_{\min}$ & $\geq (1-b)^{2}\lambda_{\min}^{2}/\lambda_{\max}$ \\ 
\hline 
\end{tabular} 
\end{table}
where $b=\max_{j\in\mathcal{C}_0,k} \bar{B}_{j,k} $ and
$\lambda_{\min}$, $\lambda_{\max}$ are the minimum /maximum
eigenvalues of $\bar{\mathbf{R}}$
\end{proposition}
\begin{proof}
Let $\bar{\mathbf{B}}_{i}$ be the $i$-th row vector of matrix $\bar{\mathbf{B}}$. 
To show the above results, recall that $\mathbf{E} =
\bar{\mathbf{B}}\mathbf{R}^{\prime}\bar{\mathbf{B}}^{\top}$. Then
\begin{align*}
& \Vert \mathbf{E}_i - \mathbf{E}_j \Vert =\Vert (\bar{\mathbf{B}}_{i} - \bar{\mathbf{B}}_{j}) \mathbf{R}^{\prime}\bar{\mathbf{B}}^{\top}\Vert \\ 
& E_{i,i}-2E_{i,j}+E_{j,j} = (\bar{\mathbf{B}}_{i} - \bar{\mathbf{B}}_{j}) \mathbf{R}^{\prime} (\bar{\mathbf{B}}_{i} - \bar{\mathbf{B}}_{j})^{\top}.
\end{align*}
It is clear that when $i,j\in\mathcal{C}_1$, i.e., they are both novel
pairs for the same ranking, $\bar{\mathbf{B}}_{i} =
\bar{\mathbf{B}}_{j}$. Hence, $\Vert \mathbf{E}_i - \mathbf{E}_j \Vert
= 0 $ and $E_{i,i}-2E_{i,j}+E_{j,j} = 0$.

When $i\in\mathcal{C}_1, j\notin \mathcal{C}_1$, we have $\bar{\mathbf{B}}_{i} = [1,0,\ldots,0]$,  $\bar{\mathbf{B}}_{j}= [\bar{B}_{j,i}, \bar{B}_{j,2},\ldots, \bar{B}_{j,K}]$ with $\bar{B}_{j,1} < 1$.
Then,
\begin{align*}
\bar{\mathbf{B}}_{i} - \bar{\mathbf{B}}_{j} & = [1-\bar{B}_{j,i}, -\bar{B}_{j,2},\ldots, -\bar{B}_{j,K}] \\
& = (1-\bar{B}_{j,i})[1,-c_2,\ldots, -c_K ] \\
& := (1-\bar{B}_{j,i}) \mathbf{e}^{\top}
\end{align*}
and $\sum_{l=2}^{K}c_{l} =1$.
Therefore, defining $\mathbf{Y} :=
\mathbf{R}^{\prime}\bar{\mathbf{B}}^{\top}$, we get
\begin{align*}
\Vert \mathbf{E}_i - \mathbf{E}_j \Vert_2 =  (1-\bar{B}_{j,i})\Vert \mathbf{Y}_1 - \sum\limits_{l=2}^{K} c_{l} \mathbf{Y}_{l}	\Vert_2
%
\end{align*}
Using the Proposition~1 in \citep{Ding13b:ref}, if $\bar{\mathbf{R}}$
is full rank with minimum eigenvalue $\lambda_{\text{min}} >0$,
then, $\bar{\mathbf{R}}$ is $\gamma$-(row)simplicial with $\gamma =
\lambda_{\text{min}}$, i.e., any row vector is at least $\gamma$
distant from any convex combination of the remaining rows. Since
$\bar{\mathbf{B}}$ is separable, $\mathbf{Y}$ is at least
$\gamma$-simplicial (see \cite{Ding14:ref} Lemma~1 ). Therefore,
\begin{align*}
\Vert \mathbf{E}_i - \mathbf{E}_j \Vert_2  \geq  (1-\bar{B}_{j,1}) \gamma \geq (1-b) \lambda_{\text{min}}
\end{align*}
where $b=\max_{j\in\mathcal{C}_0,k} \bar{B}_{j,k} < 1$. 

Similarly, note that $\Vert \mathbf{e}^{\top} \bar{\mathbf{R}}
\Vert\geq \gamma$ and let $\bar{\mathbf{R}} =
\mathbf{U}\Sigma\mathbf{U}^{\top}$ be its singular value
decomposition. If $\lambda_{\max}$ is the maximum eigenvalue of
$\bar{\mathbf{R}}$, then we have
\begin{align*}
E_{i,i}-2E_{i,j}+E_{j,j}& = (1-\bar{B}_{j,1})^{2} \mathbf{e}^{\top}\bar{\mathbf{R}}\mathbf{e} \\
& = (1-\bar{B}_{j,1})^{2} (\mathbf{e}^{\top}\bar{\mathbf{R}} ) \mathbf{U}\Sigma^{-1}\mathbf{U}^{\top} (\mathbf{e}^{\top}\bar{\mathbf{R}} )^{\top} \\
& \geq   (1-b)^{2} \lambda_{\min}^2 / \lambda_{\max}.
\end{align*} 
The inequality in the last step follows from the observation that
$\mathbf{e}^{\top}\mathbf{R}^{\prime}$ is within the column space
spanned by $\mathbf{U}$.
\end{proof}
The results in Proposition~\ref{prop:similarity} provide two
statistics for identifying novel pairs of the same topic, $\Vert
\mathbf{E}_i - \mathbf{E}_j \Vert$ and $E_{i,i}-2E_{i,j}+E_{j,j}$.
While the first is straightforward, the latter is efficient to
calculate in practice with better computational complexity.
Specifically, the set $\mathcal{J}_{i}$ in Algorithm 2 of the main paper
\begin{eqnarray*}
\mathcal{J}_{i}=\{ j: \widehat{E}_{i,i} - \widehat{E}_{i,j} -\widehat{E}_{j,i} + \widehat{E}_{j,j} \geq d/2 \}
\end{eqnarray*}
can be used to discover the set of novel pairs of the same rankings
asymptotically.
Formally,
\begin{proposition}
If $\Vert \widehat{\mathbf{E}} - \mathbf{E} \Vert_\infty \leq d/8 $, then, 
\begin{enumerate}
\item For a novel pair $i\in\mathcal{C}_k$ , $\mathcal{J}_{i} = \mathbf{C}_{k}^{c}$
\item For a non-novel pair $j\in\mathcal{C}_0$, $\mathcal{J}_{i} \supset \mathbf{C}_{k}^{c}$ 
\end{enumerate}
\end{proposition}

\subsection{Consistency of Algorithm 2 in the main paper}
Now we start to show that Algorithm 2 of the main paper can detect all
the novel pairs of the $K$ distinct rankings consistently. As a
starting point, it is straightforward to show the following result.
\begin{proposition}
\label{lem:solidangle}
Suppose $\bm{\sigma}$ is separable and $\mathbf{R}$ is full rank, then,  $q_{i} > 0$  if and only if $i$ is a  novel pair.
\end{proposition}
We denote the minimum solid angle of the $K$ extreme points by
$q_{\wedge}$.
Proposition~\ref{lem:solidangle} shows that the novel pairs can be identified by simply sorting $q_{i}$. 

The agenda is to show that the estimated solid angle in Alg.~2, 
\begin{equation}
\hat{p}_{i} = \frac{1}{P} \sum_{r=1}^{P} \mathbb{I} \lbrace \forall j\in\mathcal{J}_{i}, ~ \widehat{\mathbf{E}}_{j}  {\mathbf d}_r \leq \widehat{\mathbf{E}}_{i} {\mathbf d}_r \rbrace 
\end{equation}
converges to the ideal solid angle
\begin{align}
q_i = \Pr \lbrace \forall j \in \mathcal{S}_{i}, (\mathbf{E}_i - \mathbf{E}_j) \mathbf{d}\geq 0 \rbrace
\end{align}
hence the error event in Alg.~2 has vanishing probability as $M,P\rightarrow \infty$. 
$\mathbf{d}_1,\ldots, \mathbf{d}_P$ are iid directions drawn from a isotropic distribution. 
For a novel pair $i\in\mathcal{C}_k, k=1,\ldots, K$,  $ \mathcal{S}_{i} = \mathcal{C}_k^{c}$, and for a non-novel pair $i\in\mathcal{C}_0$, let $ \mathcal{S}_{i} = \mathcal{C}_0^{c}$.
To show the convergence of $\hat{p}_{i}$ to $p_i$, we consider an intermediate quantity, 
\begin{align*}
 p_i(\widehat{\mathbf{E}}) = \Pr \lbrace \forall j \in \mathcal{J}_{i}, (\widehat{\mathbf{E}}_i - \widehat{\mathbf{E}}_j )\mathbf{d}\geq 0   \rbrace 
\end{align*}
First, by Hoeffding's lemma, we have the following result.
\begin{proposition}
\label{prop:heofdingprop}
$\forall t\geq 0, \forall i$, 
\begin{equation}
\Pr\{\vert \hat{p}_i - p_{i}(\widehat{\mathbf{E}})\vert t \} \geq 2\exp(-2Pt^{2})
\end{equation}
\end{proposition}
Next we show the convergence of $p_{i}(\widehat{\mathbf{E}})$ to solid
angle $q_i$:
\begin{proposition}
\label{prop:solidangle}
Consider the case when $\Vert \widehat{\mathbf{E}} - \mathbf{E}\Vert_{\infty} \leq \frac{d}{8}$.
If $i$ is a novel pair, then, 
\begin{align*}
q_i - p_{i}(\widehat{\mathbf{E}}) \leq \frac{W\sqrt{W}}{\pi d_2} \Vert \widehat{\mathbf{E}}-\mathbf{E} \Vert_{\infty}
\end{align*}
Similarly, if $j$ is a non-novel pair, we have, 
\begin{align*}
 p_{j}(\widehat{\mathbf{E}}) - q_i \leq \frac{W\sqrt{W}}{\pi d_2} \Vert \widehat{\mathbf{E}}-\mathbf{E} \Vert_{\infty}
\end{align*}
where $d_2 \triangleq (1-b)\lambda_{\min}$, $d=(1-b)^{2}\lambda_{\min}^{2}/\lambda_{\max}$.
\end{proposition}
\begin{proof}

First note that, by the definition of $\mathcal{J}_{i}$ and Proposition~\ref{prop:similarity},  if $\Vert \widehat{\mathbf{E}} - \mathbf{E}\Vert_{\infty} \leq \frac{d}{8}$, 
then, for a novel pair $i \in \mathcal{C}_k$, $\mathcal{J}_{i} = \mathcal{S}(i)$. And for a non-novel pair $i\in\mathcal{C}_0$, $\mathcal{J}_{i} \supseteq \mathcal{S}(i)$.
For convenience, let
\begin{align*}
A_{j}=\{\mathbf{d}: (\widehat{\mathbf{E}}_i - \widehat{\mathbf{E}}_j )\mathbf{d}\geq 0 \} & ~~ A=\bigcap\limits_{j\in \mathcal{J}_{i}} A_j\\
B_{j} = \{\mathbf{d}: (\mathbf{E}_i - \mathbf{E}_j) \mathbf{d} \geq 0  \} & ~~ B=\bigcap\limits_{j\in \mathcal{S}(i)} B_j
\end{align*} 
For $i$ being a novel pair, we consider
\begin{align*}
q_i - p_{i}(\widehat{\mathbf{E}})  = \Pr\lbrace B \rbrace - \Pr\lbrace A \rbrace \leq \Pr\lbrace B\bigcap A^{c}\rbrace
\end{align*}
Note that $\mathcal{J}_{i} = \mathcal{S}(i)$ when $\Vert \widehat{\mathbf{E}} - \mathbf{E} \Vert \leq d/8$,  
\begin{align*}
& \Pr\lbrace B\bigcap A^{c} \rbrace  =  \Pr\lbrace B\bigcap (\bigcup\limits_{j\in\mathcal{S}(i)}A_{j}^{c}) \rbrace \\
 & \leq \sum\limits_{j\in\mathcal{S}(i)} \Pr\lbrace (\bigcap\limits_{l\in \mathcal{S}(i)} B_l)\bigcap A_{j}^{c} \rbrace   \leq \sum\limits_{j\in\mathcal{S}(i)} \Pr\lbrace  B_{j} \bigcap A_{j}^{c} \rbrace  \\
& = \sum\limits_{j\in\mathcal{S}(i)} \Pr\lbrace (\widehat{\mathbf{E}}_i - \widehat{\mathbf{E}}_j )\mathbf{d} < 0, \text{and}~  (\mathbf{E}_i - \mathbf{E}_j) \mathbf{d} \geq 0  \rbrace \\
& =  \sum\limits_{j\in\mathcal{S}(i)} \frac{\phi_{j}}{2\pi}
\end{align*}
where $\phi_{j}$ is the angle between $\mathbf{e}_{j} = \mathbf{E}_i - \mathbf{E}_j $ and $\widehat{\mathbf{e}}_{j} = \widehat{\mathbf{E}}_i - \widehat{\mathbf{E}}_j$ for any isotropic distribution on $\mathbf{d}$. 
Using the trigonometric inequality $\phi \leq \tan (\phi)$, 
\begin{align*}
\Pr\lbrace B\bigcap A^{c} \rbrace & \leq \sum\limits_{j\in\mathcal{S}(i)} \frac{\tan(\phi_{j})}{2\pi} \leq \sum\limits_{j\in\mathcal{S}(i)} \frac{1}{2\pi} \frac{\Vert \widehat{\mathbf{e}}_{j} - \mathbf{e}_{j} \Vert_2}{\Vert \mathbf{e}_{j} \Vert_2} \\
& \leq \frac{W\sqrt{W}}{\pi d_2} \Vert \widehat{\mathbf{E}}-\mathbf{E} \Vert_{\infty}
\end{align*}
where the last inequality is obtained by the relationship between the
$\ell_\infty$ norm and the$\ell_2$ norm, and the fact that for
$j\in\mathcal{S}(i)$, $\Vert \mathbf{e}_{j} \Vert_2 = \Vert
\mathbf{E}_i - \mathbf{E}_j \Vert_2 \geq d_2 \triangleq
(1-b)\lambda_{\min}$.
Therefore for a novel word $i$, we have, 
\begin{align*}
q_i - p_{i}(\widehat{\mathbf{E}}) \leq \frac{W\sqrt{W}}{\pi d_2} \Vert \widehat{\mathbf{E}}-\mathbf{E} \Vert_{\infty}
\end{align*}

Now for a non-novel word $i$,  note the fact that  $i\in\mathcal{C}_0$, $\mathcal{J}_{i} \supseteq \mathcal{S}(i)$,
\begin{align*}
p_{i}(\widehat{\mathbf{E}}) -q_{i}  = &\Pr\lbrace A \rbrace - \Pr\lbrace B \rbrace =  \Pr\lbrace A\bigcap B^{c}\rbrace \\
\leq & \sum\limits_{j\in\mathcal{S}(i)} \Pr\lbrace (\bigcap\limits_{l\in \widehat{\mathcal{S}}(i)} A_l)\bigcap B_{j}^{c} \rbrace \\
\leq & \sum\limits_{j\in\mathcal{S}(i)} \Pr\lbrace  A_{j} \bigcap B_{j}^{c} \rbrace \\
\leq & \frac{W\sqrt{W}}{\pi d_2} \Vert \widehat{\mathbf{E}}-\mathbf{E} \Vert_{\infty}
\end{align*} 
\end{proof}

A direct implication of Proposition~\ref{prop:solidangle} is,
\begin{proposition}
\label{prop:solidangleconverge}
$\forall \epsilon >0$, let $\rho =\min\{ \frac{d}{8}, \frac{\pi d_2 \epsilon}{W^{1.5}} \}$. If $\Vert \widehat{\mathbf{E}} -\mathbf{E} \Vert_\infty \leq \rho$, 
then, $q_{i} - p_{i} (\widehat{\mathbf{E}}) \leq \epsilon$ for a novel pair $i$ and $p_{j} (\widehat{\mathbf{E}}) - q_{j} \leq \epsilon$ for a non-novel pair $j$.
\end{proposition}
We now prove the consistency of Algorithm 2 of the main paper. Formally,
\begin{lemma}
\label{lem:alg2}
Algorithm 2 of the main paper can identify all the novel words from $K$ distinct rankings with error probability,
\begin{eqnarray*}
Pe \leq  2W^2\exp(-P q_{\wedge}^{2}/8)+  8W^2\exp(-\rho^{2}\eta^{4}MN/20 )
\end{eqnarray*}
where $\rho =\min\{ \frac{d}{8}, \frac{\pi d_2 q_{\wedge}}{4 W^{1.5}} \}$, $d_2 \triangleq (1-b)\lambda_{\min}$, $d=(1-b)^{2}\lambda_{\min}^{2}/\lambda_{\max}$, $b=\max_{j\in\mathcal{C}_0,k} \bar{B}_{j,k} $ and $\lambda_{\min}$, $\lambda_{\max}$ are the minimum /maximum eigenvalues of $\bar{\mathbf{R}}$. The result holds true for any isotropically distributed $\mathbf{d}$. 
\end{lemma}
\begin{proof}
First of all, we decompose the error event to be the union of the following two types,
\begin{enumerate}
\item {\it Sorting error}, i.e., $\exists i\in\bigcup_{k=1}^{K}\mathcal{C}_k, \exists j\in\mathcal{C}_0$ such that $\hat{p}_i < \hat{p}_j$. This event is denoted as $A_{i,j}$ and let $A= \bigcup A_{i,j}$. 
\item {\it Clustering error}, i.e., $\exists k, \exists i,j\in\mathcal{C}_k$ such that $i\notin\mathcal{J}_{j}$. This event is denoted as ${B}_{i,j}$ and let ${B} = \bigcup B_{i,j}$
\end{enumerate} 
According to Proposition~\ref{prop:solidangleconverge}, we also define $\rho =\min\{ \frac{d}{8}, \frac{\pi d_2 q_{\wedge}}{4 W^{1.5}} \}$ and $C = \{ \Vert \mathbf{E} -\widehat{\mathbf{E}} \Vert_\infty \geq \rho \}$.
Note that $B\subsetneq C$,

Therefore, 
\begin{eqnarray*}
Pe & =& \Pr\{ A \bigcup B \}  \\
&\leq & \Pr\{ A \bigcap C^{c} \} + \Pr\{C\} \\
&\leq & \sum_{i~novel, j~non-novel} \Pr\{ A_{i,j}\bigcap B^{c} \} + \Pr\{C\}\\
& \leq & \sum_{i,j} \Pr( \hat{p}_i - \hat{p}_j <0 \bigcap \Vert \widehat{\mathbf{E}} - \mathbf{E} \Vert_\infty \geq \rho ) \\
& &+ \Pr(\Vert \widehat{\mathbf{E}} - \mathbf{E} \Vert_\infty > \rho)
\end{eqnarray*}
The second term can be bound by Proposition~\ref{prop:convergence}. Now we focus on the first term. Note that
\begin{eqnarray*}
\hat{p}_i -  \hat{p}_j &=& \hat{p}_i - \hat{p}_j - p_i(\widehat{\mathbf{E}})+ p_i(\widehat{\mathbf{E}}) \\
& & - q_i + q_i  - p_j(\widehat{\mathbf{E}})+p_j(\widehat{\mathbf{E}})-q_j+q_j \\
& = & \{ \hat{p}_i - p_i(\widehat{\mathbf{E}}) \} + \{ p_i(\widehat{\mathbf{E}}) - q_i \} \\
& & + \{ p_j(\widehat{\mathbf{E}})- \hat{p}_j \} + \{q_j - p_j(\widehat{\mathbf{E}}) \} \\
& & + q_i - q_j
\end{eqnarray*}
and the fact that $q_i - q_j \geq q_{\wedge}$, then,, 
\begin{eqnarray*}
&&\Pr( \hat{p}_i < \hat{p}_j \bigcap \Vert \widehat{\mathbf{E}} - \mathbf{E} \Vert_\infty \leq \rho ) \\
&\leq & \Pr( p_i(\widehat{\mathbf{E}})- \hat{p}_i \geq q_{\wedge}/4) + \Pr( \hat{p}_j - p_j(\widehat{\mathbf{E}}) \geq q_{\wedge}/4  ) \\
& & + \Pr(q_i - p_i(\widehat{\mathbf{E}})\geq q_{\wedge}/4)\bigcap \Vert \widehat{\mathbf{E}} - \mathbf{E} \Vert_\infty \leq \rho   ) \\
& & + \Pr( p_j(\widehat{\mathbf{E}})-q_j\geq q_{\wedge}/4)\bigcap \Vert \widehat{\mathbf{E}} - \mathbf{E} \Vert_\infty \leq \rho   ) \\
&\leq & 2\exp(-P q_{\wedge}^{2}/8) \\
& &+ \Pr(q_i - p_i(\widehat{\mathbf{E}})\geq q_{\wedge}/4)\bigcap \Vert \widehat{\mathbf{E}} - \mathbf{E} \Vert_\infty \leq \rho   ) \\
& & + \Pr( p_j(\widehat{\mathbf{E}})-q_j\geq q_{\wedge}/4)\bigcap \Vert \widehat{\mathbf{E}} - \mathbf{E} \Vert_\infty \leq \rho   )
\end{eqnarray*}
The last equality is by Proposition~\ref{prop:heofdingprop}.
For the last two terms, by Proposition \ref{prop:solidangleconverge} is 0. Therefore, applying Lemma~\ref{lem:converge} we obtain,
\begin{eqnarray*}
Pe \leq  2W^2\exp(-P q_{\wedge}^{2}/8)+  8W^2\exp(-\rho^{2}\eta^{4}MN/20 )
\end{eqnarray*}
\end{proof}

\subsection{Consistency of algorithm 3}
Now we show that Algorithm 3 and 4 of the main paper can consistently estimate the ranking matrix $\bm{\sigma}$, given the success of the Algorithm 2. 
Without loss of generality, let $1,\ldots, K$ be the novel pairs of $K$ distinct rankings. 
We first show that the solution of the constrained linear regression is consistent:
\begin{proposition}
\label{prop:optimizationconverge}
The solution to the following optimization problem 
\begin{align*}
\widehat{\mathbf{b}}^{*} =  \arg\min_{b_j \geq 0, \sum b_j =1} \Vert \widehat{\mathbf{E}}_i -\sum\limits_{j=1}^{K} b_j \widehat{\mathbf{E}}_j\Vert
\end{align*}
converges to the $i$-th row of $\bar{\mathbf{B}}$, $\bar{\mathbf{B}}_{i}$, as $M\rightarrow\infty$.
Moreover,
\begin{align*}
\Pr ( \Vert \widehat{\mathbf{b}}^{*} -  \bar{\mathbf{B}}_{i}\Vert_\infty \geq \epsilon )\leq 8W^2 \exp(- \frac{\epsilon^2 MN \lambda_{\min}\eta^4}{80W^{0.5}})
\end{align*}
\end{proposition}
\begin{proof}
We note that $\bar{\mathbf{B}}_{i}$ is the optimal solution to the following problem
\begin{align*}
{\mathbf{b}}^{*} =  \arg\min_{b_j \geq 0, \sum b_j =1} \Vert {\mathbf{E}}_i -\sum\limits_{j=1}^{K} b_j {\mathbf{E}}_j\Vert
\end{align*}
Define $f(\mathbf{E},\mathbf{b}) = \Vert {\mathbf{E}}_i
-\sum\limits_{j=1}^{K} b_j {\mathbf{E}}_j\Vert$ and note the fact that
$f(\mathbf{E},\mathbf{b}^{*}) =0$. Let $\mathbf{Y} =
[\mathbf{E}_1^{\top}, \ldots, \mathbf{E}_K^{\top} ]^{\top}$. Then,
\begin{align*}
& f(\mathbf{E},\mathbf{b}) - f(\mathbf{E},\mathbf{b}^{*})  =\Vert {\mathbf{E}}_i -\sum\limits_{j=1}^{K} b_j {\mathbf{E}}_j \Vert -0 \\
=& \Vert \sum\limits_{j=1}^{K} (b_j - b_j^{*}) {\mathbf{E}}_j \Vert =\sqrt{ (\mathbf{b} - \mathbf{b}^{*}) \mathbf{Y Y^{\top}} (\mathbf{b} - \mathbf{b}^{*})^{\top} }\\
\geq & \Vert \mathbf{b} - \mathbf{b}^{*} \Vert\lambda_{\text{min}}
\end{align*}
where $\lambda_{\min}>0$ is the minimum eigenvalue of $\bar{\mathbf{R}}$.
Next, note that,
\begin{align*}
 \vert f(\mathbf{E},\mathbf{b}) - f(\widehat{\mathbf{E}},\mathbf{b}) \vert \leq & \Vert \mathbf{E}_i -\widehat{\mathbf{E}}_i + \sum b_j (\widehat{\mathbf{E}}_j - \mathbf{E}_j) \Vert \\
\leq & \Vert \mathbf{E}_i -\widehat{\mathbf{E}}_i \Vert + \sum b_j \Vert \widehat{\mathbf{E}}_j - \mathbf{E}_j \Vert \\
\leq & 2 \max_{w} \Vert \widehat{\mathbf{E}}_w - \mathbf{E}_w \Vert
\end{align*}
Combining the above inequalities, we obtain,  
\begin{align*}
\Vert \widehat{\mathbf{b}}^{*} - \mathbf{b}^{*} \Vert \leq & \frac{1}{\lambda_{\min}} \lbrace f(\mathbf{E},\widehat{\mathbf{b}}^{*}) - f(\mathbf{E},{\mathbf{b}}^{*}) \rbrace \\
= &   \frac{1}{\lambda_{\min}} \lbrace f(\mathbf{E},\widehat{\mathbf{b}}^{*}) -f(\widehat{\mathbf{E}},\widehat{\mathbf{b}}^{*}) + f(\widehat{\mathbf{E}},\widehat{\mathbf{b}}^{*}) \\
& ~~ - f(\widehat{\mathbf{E}},{\mathbf{b}}^{*}) +f(\widehat{\mathbf{E}},{\mathbf{b}}^{*})  - f(\mathbf{E},{\mathbf{b}}^{*}) \rbrace \\
\leq &  \frac{1}{\lambda_{\min}} \lbrace f(\mathbf{E},\widehat{\mathbf{b}}^{*}) -f(\widehat{\mathbf{E}},\widehat{\mathbf{b}}^{*}) \\
& ~~ +f(\widehat{\mathbf{E}},{\mathbf{b}}^{*})  - f(\mathbf{E},{\mathbf{b}}^{*}) \rbrace \\
\leq & \frac{4 W^{0.5}}{\lambda_{\min}}  \Vert \widehat{\mathbf{E}} - \mathbf{E} \Vert_\infty
\end{align*}
where the last term converges to $0$ almost surely. The convergence rate follows directly from Lemma~\ref{lem:converge}.
\end{proof}

Now for the row-scaling step in algorithm 3,
\begin{align}
\label{eq:rowscale}
\nonumber
\widehat{\mathbf{B}}_{i}& := \hat{\mathbf{b}}^{*} (i)^{\top} (\frac{1}{M}\mathbf{X}\mathbf{1}_{M\times 1})\\
& \rightarrow \bar{\mathbf{B}}_{i} (\mathbf{B}_{i}\mathbf{a})  = \mathbf{B}_i \diag(\mathbf{a}) 
\end{align}
We point out that the ``column-normalization'' step in
\cite{Ding14:ref} which was used to get rid of the $\diag(\mathbf{a})$
component in the above equation is not necessary in our approach.
To show the convergence rate of the above equation, it is straightforward to apply the result in Lemma~\ref{lem:converge}
\begin{proposition}
\label{prop:rowscaling}
For the row-scaled estimation $\hat{\mathbf{B}}_{i}$ as in Eq.~\eqref{eq:rowscale}, we have,
\begin{equation*}
\Pr( \vert \hat{\mathbf{B}}_{i,k} - \mathbf{B}_{i,k}a_k \vert \geq \epsilon ) \leq 8W^2 \exp(- \frac{\epsilon^2 MN \lambda_{\min}\eta^4}{160W^{0.5}})
\end{equation*}
\end{proposition}
\begin{proof}
By Proposition~\ref{prop:optimizationconverge}, we have,
\begin{align*}
\Pr ( \vert \widehat{\mathbf{b}}^{*}(i)_{k} -  \bar{\mathbf{B}}_{i,k}\vert \geq \epsilon/2 )\leq 8W^2 \exp(- \frac{\epsilon^2 MN \lambda_{\min}\eta^4}{160W^{0.5}})
\end{align*}
Recall that,
\begin{align*}
\Pr ( \vert \frac{1}{M}\mathbf{X}\mathbf{1}_{M\times 1} - \mathbf{B}_{i}\mathbf{a} \vert \geq \epsilon/2)\leq \exp(-\epsilon^2 MN/2)
\end{align*}
Therefore, 
\begin{align*}
& \Pr( \vert \hat{\mathbf{B}}_{i,k} - \mathbf{B}_{i,k}a_k \vert  \geq \epsilon ) \\
 \leq & 8W^2 \exp(- \frac{\epsilon^2 MN \lambda_{\min}\eta^4}{80W^{0.5}}) + \exp(-\epsilon^2 MN/2) \\
\end{align*}
where the second term is dominated by the first term. 
\end{proof}

For the rest of this section, we will use $(i,j)$ to index the $W$ rows of $\mathbf{E},\mathbf{B},\bm{\sigma}$. 
Recall in Eq.~\eqref{eq:rowscale}, $\widehat{\mathbf{B}}_{(i,j),k} \rightarrow \mathbf{B}_{(i,j),k} a_{k} = \mu_{i,j} \sigma_{(i,j),k} a_{k}$, and $\widehat{\mathbf{B}}_{(j,i),k} \rightarrow \mathbf{B}_{(j,i),k} a_{k} = \mu_{i,j} \sigma_{(j,i),k} a_{k}$, and in algorithm 1 of the main paper, we consider
\begin{align*}
\widehat{\sigma}_{(i,j), k} \leftarrow & \frac{ \widehat{\mathbf{B}}_{(i,j),k} }{\widehat{\mathbf{B}}_{(i,j),k} + \widehat{\mathbf{B}}_{(j,i),k}} \\
 \doteq & \frac{ \sigma_{(i,j),k}\mu_{i,j} a_k }{ \sigma_{(i,j),k}\mu_{i,j} a_k + \sigma_{(j,i),k}\mu_{i,j} a_k }
\end{align*}
Therefore, due to the rounding scheme of the last step,  the estimation is consistent if $\vert \widehat{\mathbf{B}}_{(i,j),k} - \mathbf{B}_{(i,j),k} a_{k}\vert \leq 0.5\mu_{i,j}a_k$. $\eta$ is a lower bound of $\mu_{i,j}a_k$. 
Putting the above results together, we have,
\begin{lemma}
\label{lem:alg3}
Given the success in Lemma~\ref{lem:alg2}, Algorithm 3 and the remaining post-processing steps in Algorithm 1 of the main paper can consistently estimate the ranking matrix $\bm{\sigma}$ as $M\rightarrow\infty$. Moreover, the error probability is less than $8W^2 \exp(- \frac{MN \lambda_{\min}\eta^6}{160W^{0.5}})$.
\end{lemma}
\subsection{Proof of Theorem 2}
We now formally prove the sample complexity Theorem 2 in the main paper.

{\bf Theorem 2}
Let $\bm{\sigma}$ be separable and $\mathbf{R}$ be full rank. Then the
overall Algorithm~1 consistently recovers $\bm{\sigma}$ up to a column permutation as the number of users $M\rightarrow \infty$ and number of projections $P\rightarrow \infty$.
Furthermore, $\forall \delta>0$, if
\begin{align*}
M \geq \max \Biggl\{   40 \frac{\log(3W/\delta)}{ N \rho^2 \eta^4}, ~
320 \frac{W^{0.5} \log(3W/\delta)}{N \eta^6 \lambda_{\min}} \Biggr\}
\end{align*}
and for 
\begin{equation*}
\numofproj \geq 16 \frac{\log(3W/\delta)}{q_{\wedge}^2}
\end{equation*}
then Algorithm~1 fails with probability at most $\delta$.  The other model parameters are
defined as  $\eta = \min_{1\leq w \leq W} [\mathbf{B}\mathbf{a}]_w$,  
$\rho =\min\{ \frac{d}{8}, \frac{\pi d_2 q_{\wedge}}{4 W^{1.5}} \}$, $d_2 \triangleq (1-b)\lambda_{\min}$, $d=(1-b)^{2}\lambda_{\min}^{2}/\lambda_{\max}$, $b=\max_{j\in\mathcal{C}_0,k} \bar{B}_{j,k} $ and $\lambda_{\min}$, $\lambda_{\max}$ are the minimum /maximum eigenvalues of $\bar{\mathbf{R}}$.
$q_{\wedge}$ is the minimum normalized solid angle of the extreme
points of the convex hull of the rows of $\mathbf{E}$.
\begin{proof}
We combine the results in Lemmas~\ref{lem:alg2} and \ref{lem:alg3}, i.e., the error probability of alg.~1 can be upper bounded by 
\begin{align*}
Pe \leq & 2W^2\exp(-P q_{\wedge}^{2}/8)+  8W^2\exp(-\rho^{2}\eta^{4}MN/20 ) \\
 & + 8W^2 \exp(- \frac{MN \lambda_{\min}\eta^6}{160W^{0.5}})
\end{align*}
This leads to the sample complexity results in the theorem.
\end{proof}
\section{Algorithm 2 and Theorem 2 for Gaussian Random Directions}
The proof in Section~\ref{sec:thm2} holds for any isotropic distribution on $\mathbf{d}$. 
If we assume $\mathbf{d}$ to be the standard spherical Gaussian distribution, we can have better sample complexity bounds following the steps in \citep[][Theorem 2]{Ding14:ref}. 
First note that,  
\begin{proposition}
\label{key-claim}
Let ${\mathbf X}^{n} ,\mathbf{X} \in\rR^{m}$ be two random vectors, $\mathbf{a},\bm{\epsilon} \in\rR^{m}$ be two vectors and $\bm{\epsilon}>\mathbf{0}$.
\begin{align*}
& \lvert \Pr\lbrace\mathbf {X}^{n}\leq \mathbf{a}\rbrace - \Pr\lbrace\mathbf{X}\leq \mathbf{a}\rbrace \rvert \\
 \leq & \Pr (\exists i : \vert {X}_i^{n} -{X}_i \vert \geq \epsilon_i) + \Pr( \mathbf{a}-\bm{\epsilon} \leq \mathbf{X} \leq \mathbf{a} + \bm{\epsilon})
\end{align*}
The inequality is element-wise.
\end{proposition}
\begin{proof}
Note that
\begin{align*}
\Pr\lbrace\mathbf {X}^{n}\leq  \mathbf{a}\rbrace \leq & \Pr\lbrace \mathbf {X}^{n}\leq \mathbf{a}, \forall i: \vert {X}^{n}_i - X_i \vert \leq \epsilon_i  \rbrace \\
& + \Pr\lbrace \mathbf {X}^{n}\leq \mathbf{a}, \exists i: \vert {X}^{n}_i - X_i \vert \geq \epsilon_i  \rbrace \\
\leq & \Pr\lbrace \mathbf{X}\leq \mathbf{a}+\bm{\epsilon} \rbrace + \Pr\lbrace  \exists i: \vert {X}^{n}_i - X_i \vert \geq \epsilon_i  \rbrace
\end{align*}
Similarly, by swapping $\mathbf{X}^{n}$ and $\mathbf{X}$, we have, 
\begin{align*}
\Pr\lbrace\mathbf {X}\leq  \mathbf{a}-\bm\epsilon \rbrace \leq  \Pr\lbrace \mathbf{X}^{n}\leq \mathbf{a} \rbrace + \Pr\lbrace  \exists i: \vert {X}^{n}_i - X_i \vert \geq \epsilon_i  \rbrace
\end{align*}
Combining them concludes the proof. 
\end{proof}
\begin{proposition}
\label{prop:solidangleGuassian}
Let the random projection directions be $\mathbf{d} \sim  \mathcal{N}(\mathbf{0},\mathbf{I}_{W})$ in Algorithm 2 of the main paper. Then, $\forall ~ \epsilon > 0$, 
let $\rho =\min\{ \frac{d}{8}, \frac{\sqrt{\pi}\epsilon  d_2 }{4 K \sqrt{W \log(2W/\epsilon)}}\}$. If $\Vert \widehat{\mathbf{E}} -\mathbf{E} \Vert_\infty \leq \rho$, 
then, $q_{i} - p_{i} (\widehat{\mathbf{E}}) \leq \epsilon$ for a novel pair $i$ and $p_{j} (\widehat{\mathbf{E}}) - q_{j} \leq \epsilon$ for a non-novel pair $j$.
\end{proposition}
\begin{proof}
Recall the definition of $q_i$ and $p_i(\widehat{\mathbf{E}})$,
\begin{eqnarray*}
q_i &= &\Pr \lbrace \forall j \in \mathcal{S}(i),~  \mathbf{E}_i \mathbf{d} \geq \mathbf{E}_j \mathbf{d}  \rbrace \\
p_{i} (\widehat{\mathbf{E}}) &= &\Pr \lbrace \forall j \in\mathcal{J}_{i}, ~ \widehat{\mathbf{E}}_i \mathbf{d} \geq \widehat{\mathbf{E}}_i \mathbf{d} \rbrace 
\end{eqnarray*}
When $i$ is a novel word, $\mathcal{S}(i) = \mathcal{J}_{i}$ for $\Vert \widehat{\mathbf{E}} -\mathbf{E} \Vert_\infty \leq \rho\leq d/8$, therefore, by Proposition~\ref{key-claim}, we have,
\begin{equation}
\begin{split}
\label{eq:decomp}
 \vert q_{i} - p_{i} (\widehat{\mathbf{E}}) \vert  &\leq  \Pr  (\exists j\in\mathcal{J}_{i}  : \vert \mathbf{e}_{i,j} \mathbf{d} \vert \geq \delta) \\
& + \Pr ( \forall j\in \mathcal{J}_{i} : \vert \mathbf{z}_{ij} {\mathbf d} \vert \leq \delta ) 
\end{split} 
\end{equation}
where $\mathbf{e}_{i,j} = {\mathbf E}_i-\widehat{\mathbf{E}}_i +\widehat{\mathbf{E}}_j-{\mathbf E}_j$ and $\mathbf{z}_{ij} = \mathbf{E}_i - \mathbf{E}_j$. 
To apply the union bound to the second term in Eq.~\eqref{eq:decomp},
it suffice to consider only $j\in\bigcup_{k=1}^{K}
\mathcal{C}_{k}$. Therefore, by union bounding both the first and second
terms, we obtain,
\begin{equation*}
\begin{split}
\label{eq:GaussianDecomp}
& \vert q_{i} - p_{i} (\widehat{\mathbf{E}}) \vert \\
\leq &\sum_{j} \Pr (\vert \mathbf{e}_{i,j} \mathbf{d} \vert \geq \delta) + \sum_{j}\Pr (\vert \mathbf{z}_{ij} {\mathbf d} \vert \leq { \delta}) 
\end{split}
\end{equation*}
Note that $\mathbf{e}_{ij}\mathbf{d} \sim \mathcal{N}
(0,\lVert\mathbf{z}_{ij} \rVert_2^2)$ and $\mathbf{z}_{ij}\mathbf{d}
\sim \mathcal{N} (0,\lVert\mathbf{a}_{ij} \rVert_2^2)$ conditioned on
$\widehat{\mathbf{E}}$. Using the properties of the Gaussian
distribution we have,
\begin{equation*}
\begin{split}
\Pr  (\vert\mathbf{z}_{ij}\mathbf{d}\vert  \leq \delta ) & = \int_{-\delta}^{\delta} \frac{1}{\sqrt{2\pi}\Vert \mathbf{z}_{ij} \Vert}e^{-t^2/2\Vert \mathbf{z}_{ij}\Vert^2} dt \\
& \leq \frac{\sqrt{2/\pi}}{\Vert \mathbf{z}_{ij}\Vert} \delta
\end{split}.
\end{equation*} 
By Proposition \ref{prop:similarity}, $\Vert\mathbf{z}_{ij}\Vert \geq d_2$ for $j\in\mathcal{J}_{i}$, therefore, $\Pr  (\vert\mathbf{z}_{ij}\mathbf{d}\vert  \leq \delta ) \leq  \frac{\sqrt{2/\pi}}{d_2} \delta $. Similarly, note that 
\begin{equation*}
\Pr  ( \vert \mathbf{e}_{i,j}\mathbf{d} \vert \geq \delta |\widehat{\mathbf{E}} ) = 2Q(\delta/\Vert \mathbf{e}_{i,j} \Vert)\leq \exp(-\delta^2/2 \Vert \mathbf{e}_{i,j} \Vert^2)
\end{equation*} 
by the property of the $Q$-function.
Note that 
\begin{align*}
\Vert \mathbf{e}_{i,j} \Vert \leq &\Vert {\mathbf E}_i-\widehat{\mathbf{E}}_i\Vert +\Vert \widehat{\mathbf{E}}_j-{\mathbf E}_j\Vert \\
\leq & 2 W^{0.5}\Vert \mathbf{E} - \widehat{\mathbf{E}} \Vert_\infty
\end{align*}
Then, by marginalizing over $\widehat{\mathbf{E}}$ we obtain, $\Pr  ( \vert \mathbf{e}_{i,j}\mathbf{d} \vert \geq \delta ) \leq \exp(-\delta^2/8W \Vert \mathbf{E} - \widehat{\mathbf{E}}  \Vert_\infty^2)$. Combining these results, we obtain,
\begin{equation*}
\vert q_i - p_i(\widehat{\mathbf{E}}) \vert \leq K \frac{\sqrt{2/\pi}}{ d_2 }\delta + W \exp(-\delta^2/8W \Vert \mathbf{E} - \widehat{\mathbf{E}}  \Vert_\infty^2) 
\end{equation*}
hold true for any $\delta > 0$.
Therefore, if we set $\delta = \frac{\epsilon_0\rho}{2K\sqrt{2/\pi}}$,
and require 
\begin{align*}
\Vert \mathbf{E} - \widehat{\mathbf{E}}  \Vert_\infty \leq \frac{\sqrt{\pi}\epsilon  d_2 }{4 K \sqrt{W \log(2W/\epsilon)}}
\end{align*}
then $\vert q_i - p_i(\widehat{\mathbf{E}}) \vert\leq \epsilon$.  In
summary, we require $\Vert \mathbf{E} - \widehat{\mathbf{E}} \Vert_\infty
\leq \min\{\frac{\sqrt{\pi}\epsilon d_2 }{4 K \sqrt{W
    \log(2W/\epsilon)}}, d/8 \}$. We note that the argument above holds
true for a non-novel pair as well.
\end{proof}

In Proposition~\ref{prop:solidangleGuassian}, the bound on $\Vert \mathbf{E} - \widehat{\mathbf{E}}  \Vert_\infty$ is,
\begin{align*}
\min\{\frac{d}{8}, \frac{\sqrt{\pi}\epsilon  d_2 }{4 K \sqrt{W \log(2W/\epsilon)}}\}
\end{align*}
which is an improvement over the result in
Proposition~\ref{prop:solidangleconverge},
\begin{align*}
\min\{ \frac{d}{8}, \frac{\pi d_2 \epsilon}{W^{1.5}} \}
\end{align*}
where we could reduce the dependence on $W$ from $W\sqrt{W}$ to
$K\sqrt{W}$. Since $K\ll W$, we obtain a gain over the general
isotropic distribution. This leads to lightly improved results for the
overall sample complexity bounds:

{\bf Theorem 2}(Gaussian Random Projections)
Let $\bm{\sigma}$ be separable and $\mathbf{R}$ be full rank. Then the
overall Algorithm~1 consistently recovers $\bm{\sigma}$ up to a column
permutation as the number of users $M\rightarrow \infty$ and number of
projections $P\rightarrow \infty$.
Furthermore, if the random directions for projections are drawn from a
spherical Gaussian distribution, then $\forall \delta>0$, if
\begin{align*}
M \geq \max \Biggl\{   40 \frac{\log(3W/\delta)}{ N \rho^2 \eta^4}, ~
320 \frac{W^{0.5} \log(3W/\delta)}{N \eta^6 \lambda_{\min}} \Biggr\}
\end{align*}
and for 
\begin{equation*}
\numofproj \geq 16 \frac{\log(3W/\delta)}{q_{\wedge}^2}
\end{equation*}
then Algorithm~1 fails with probability at most $\delta$.  The other
model parameters are defined as $\eta = \min_{1\leq w \leq W}
[\mathbf{B}\mathbf{a}]_w$, $\rho =\min\{\frac{d}{8}, \frac{\sqrt{\pi}
  d_2 q_{\wedge} }{4 K \sqrt{W \log(2W/q_{\wedge})}}\}$, $d_2
\triangleq (1-b)\lambda_{\min}$,
$d=(1-b)^{2}\lambda_{\min}^{2}/\lambda_{\max}$,
$b=\max_{j\in\mathcal{C}_0,k} \bar{B}_{j,k} $ and $\lambda_{\min}$,
$\lambda_{\max}$ are the minimum /maximum eigenvalues of
$\bar{\mathbf{R}}$.  $q_{\wedge}$ is the minimum normalized solid
angle of the extreme points of the convex hull of the rows of
$\mathbf{E}$.

\section{Proof of Theorem 1}
The stated computational efficiency can be achieved in the same way as
discussed in Proposition~1 and~2 in \cite{Ding14:ref}. We need to point out that the post-processing steps in Algorithm 4 requires a computation time of $\mathcal{O}(WK)$ which is dominated by that of the Algorithm 2 and 3.

\bibliographystyle{abbrvnat}
\footnotesize
\bibliography{ref}

\end{document}